\documentclass[10pt]{article} 
\usepackage[preprint]{tmlr}

\usepackage[hyphens]{url}
\usepackage[pdftex]{graphicx}
\usepackage{natbib}
\usepackage{caption}
\usepackage{hyperref}
\usepackage{booktabs}
\usepackage{amssymb, amsmath, amsthm, mathrsfs, mathtools}
\usepackage[mathscr]{euscript}
\usepackage{bm}  
\usepackage{xcolor}
\usepackage{subcaption}
\usepackage{outlines}
\usepackage{multirow}
\usepackage{amsmath, amssymb}
\let\AND\relax
\usepackage{algorithm}
\usepackage{algorithmic}
\usepackage{booktabs}
\usepackage{adjustbox}
\usepackage{enumitem}
\usepackage[font=small,skip=2pt]{caption}
\usepackage{booktabs}
\usepackage{xspace}

\newcommand{\fidp}{\mathsf{Fid}^+}
\newcommand{\fidm}{\mathsf{Fid}^-}
\newcommand{\spars}{\mathsf{Sparsity}}
\newcommand{\cons}{\mathsf{Consensus}}
\newcommand{\ri}{\mathsf{R.I.}}

\newcommand{\eq}[1]{
    \begin{align}
        #1
    \end{align}
}

\newcommand{\picc}[1]{Fig.~\ref{fig:#1}}

\newcommand{\tabc}[1]{Table~\ref{tab:#1}}

\newcommand{\algc}[1]{Algorithm~\ref{alg:#1}}

\newcommand{\eqc}[1]{Eq.~(\ref{eq:#1})}

\newcommand{\secc}[1]{Section~\ref{sec:#1}}

\newcommand{\bfa}[1]{\textbf{#1}}

\newcommand{\ol}[1]{ 
    \begin{outline}[enumerate]
    #1
    \end{outline}
}

\definecolor{brightorange}{RGB}{255,127,0} 

\definecolor{tdbg}{HTML}{FFCC00}
\definecolor{fixbg}{HTML}{CC0000}
\definecolor{revbg}{HTML}{EB0CB7}
\definecolor{checkbg}{HTML}{95FA2F}
\definecolor{changebg}{HTML}{12C915}
\definecolor{ongoingbg}{HTML}{19FFA5}
\definecolor{ablatebg}{HTML}{1F3FF5}
\definecolor{modbg}{HTML}{CC00CC}
\definecolor{tdlbg}{HTML}{BFFFF0}
\definecolor{tdnbg}{HTML}{FFE3BF}
\definecolor{notebg}{HTML}{FF0000}
\definecolor{donebg}{HTML}{00C3FF}
\definecolor{textdark}{HTML}{1F1F1F}

\newcommand{\cA}{{\mathcal{A}}}

\newcommand{\cG}{\mathcal{G}}

\newcommand{\cV}{\mathcal{V}}
\newcommand{\cE}{\mathcal{E}}

\newcommand{\R}{\ensuremath{\mathbb R}}

\newcommand{\cP}{\EuScript{P}}

\newcommand{\best}[1]{\bm{#1}} 
\newcommand{\sbest}[1]{\underline{#1}} 

\newtheorem{theorem}{Theorem} 

\newtheorem{corollary}[theorem]{Corollary}

\author{
    \name\!\!Haribandhu Jena \email haribandhu.jena@niser.ac.in\\
    \addr School of Computer Sciences\\
    National Institute of Science Education and Research\\
    An OCC of Homi Bhabha National Institute, India\\\\
    \AND
    \name Jyotirmaya Shivottam \email jyotirmaya.shivottam@niser.ac.in\\
    \addr School of Computer Sciences\\
    National Institute of Science Education and Research\\
    An OCC of Homi Bhabha National Institute, India\\\\
    \AND
    \name Subhankar Mishra \email smishra@niser.ac.in\\
    \addr School of Computer Sciences\\
    National Institute of Science Education and Research\\
    An OCC of Homi Bhabha National Institute, India
}

\usepackage{graphicx}
\usepackage{amsmath,amssymb}
\usepackage{booktabs}
\usepackage{hyperref}
\usepackage{url}
\usepackage{xcolor}

\definecolor{customblue}{HTML}{2e47be}
\definecolor{hidden-blue}{RGB}{215,238,249}
\definecolor{hidden-black}{RGB}{20,68,106}
\definecolor{takeaway-blue}{RGB}{218,227,243}
\definecolor{takeaway-title-blue}{RGB}{51,74,133}
\definecolor{authorYellow}{HTML}{ff9100}
\definecolor{authorGreen}{HTML}{59C4B8}
\definecolor{authorBlue}{HTML}{4C9CFF}
\definecolor{authorPurple}{HTML}{AF73FA}
\definecolor{authorPink}{HTML}{f0539b}
\definecolor{authorCyan}{HTML}{548c29}

\hypersetup{
    colorlinks=true,
    citecolor=customblue,
    linkcolor=customblue,
    urlcolor=customblue 
}
\title{QGraphLIME - Explaining Quantum Graph Neural Networks}

\date{\today}

\begin{document}
\maketitle

\begin{abstract}
Quantum graph neural networks offer a powerful paradigm for learning on graph-structured data, yet their explainability is complicated by measurement-induced stochasticity and the combinatorial nature of graph structure. In this paper, we introduce \textbf{QuantumGraphLIME (QGraphLIME)}, a model-agnostic, post-hoc framework that treats model explanations as distributions over local surrogates fit on structure-preserving perturbations of a graph. By aggregating surrogate attributions together with their dispersion, QGraphLIME yields uncertainty-aware node and edge importance rankings for quantum graph models. The framework further provides a distribution-free, finite-sample guarantee on the size of the surrogate ensemble: a Dvoretzky-Kiefer-Wolfowitz bound ensures uniform approximation of the induced distribution of a binary class probability at target accuracy and confidence under standard independence assumptions. Empirical studies on controlled synthetic graphs with known ground truth demonstrate accurate and stable explanations, with ablations showing clear benefits of nonlinear surrogate modeling and highlighting sensitivity to perturbation design. Collectively, these results establish a principled, uncertainty-aware, and structure-sensitive approach to explaining quantum graph neural networks, and lay the groundwork for scaling to broader architectures and real-world datasets, as quantum resources mature. Code is available at \href{https://github.com/smlab-niser/qglime}{https://github.com/smlab-niser/qglime}.
\end{abstract}

\section{Introduction}
Quantum Neural Networks (QNNs) mark a shift in computational paradigm, leveraging quantum superposition and entanglement to form rich data representations that can outperform classical models in expressivity under certain resource constraints~\citep{qnnrev}. However, their outputs are inherently probabilistic due to quantum measurement, introducing randomness that complicates model explanations~\citep{qlime}. Graph Neural Networks (GNNs) have proven highly effective for structured data across various graph learning tasks such as node classification, edge prediction, or graph classification, by exploiting graph topology~\citep{gnnsurv}. Merging QNNs with GNN design principles to produce Quantum Graph Neural Networks (QGNNs) promises a more efficient approach for graph learning tasks by leveraging quantum computing advantages~\citep{verdon,skolkik,eqgc,qgnnrev}, but it also leads to explainability challenges, as existing classical explanation methods for GNNs~\citep{xgnnrev} cannot account for quantum state collapse, measurement noise, and graph complexity simultaneously. Without robust explainability and uncertainty quantification techniques, decisions based on quantum graph neural networks remain opaque and unsuitable for high-stakes applications, risking their adoption in regulated domains such as drug discovery and medical device development.

To address these challenges, model-agnostic explanation methods can be particularly valuable, as they can be applied across diverse architectures without dependence on model parameters~\citep{xaidoshi}. Several of such methods explain complex neural networks by locally approximating them with simpler \emph{surrogate} models such as logistic regression, thus enabling interpretable insights while preserving methodological flexibility. Local Interpretable Model-Agnostic Explanations (LIME) fits sparse linear models on perturbed versions of an input to reveal feature importance for the specific prediction, corresponding to that input~\citep{lime}. SHapley Additive exPlanations (SHAP) uses game-theoretic Shapley values to attribute each feature's contribution to the prediction \citep{shap}. For graph neural networks, GraphLIME adapts LIME with Hilbert-Schmidt Independence Criterion (HSIC) Lasso to select nonlinear combinations of node features and subgraph patterns that best explain model predictions~\citep{glime}. StGraphLIME~\citep{stglime} extends this with Group Lasso~\citep{groupl} to identify key nodes and connected substructures, improving explanation quality on graph and graph-series classification. In quantum ML, Q-LIME employs an ensemble of surrogates to accommodate measurement randomness in parameterized quantum circuits (PQCs), averaging explanations over repeated circuit runs to get a Region of Indecision for quantifying the QNN's uncertainty ~\citep{qlime}.

However, the discussed methods are either tailored for classical models or restricted to non-graph quantum neural networks. To date, no dedicated framework exists for providing explanations specifically for quantum graph neural networks; the field lacks model-agnostic or post-hoc techniques that can simultaneously account for quantum measurement noise and maintain graph structural coherence in explanations. To evaluate the robustness of these explanations in the presence of quantum measurement randomness, we introduce an ensemble-based evaluation procedure. In this approach, the trained QGNN is repeatedly queried on perturbed versions of the input graph, generating corresponding prediction labels for each input. These perturbed inputs and labels are then passed to an HSIC-based non-linear surrogate to compute feature or substructure-level explanation scores. By repeating this procedure across an ensemble of multiple surrogates, we capture the variability introduced by quantum stochasticity and quantify the explanation stability. Our approach aligns conceptually with Q-LIME's probabilistic local interpretability, and following its integration of classical data with a quantum model, we adopt a similar strategy by operating on classical graph data with a QGNN.

\paragraph{Contributions.}
\ol{
    \1 We present \bfa{QuantumGraphLIME (QGraphLIME)}, a \emph{model-agnostic}, post-hoc explainer that treats explanations as \emph{distributions} over local surrogates fit on structure-preserving graph perturbations, yielding uncertainty-aware node/edge rankings.
    \1 We establish a \emph{distribution-free} finite-sample guarantee on the \emph{minimum} surrogate ensemble size via the Dvoretzky--Kiefer--Wolfowitz bound, with a simultaneous multi-graph/multi-statistic extension by the union bound.
    \1 We leverage \emph{nonlinear HSIC-based surrogates} (HSIC-L1, HSIC-Group) to capture graph-dependent, group-structured dependencies, producing sparse, coherent attributions compatible with general quantum graph neural networks, which we validate with equivariant quantum graph models.
    \1 We introduce \emph{ensemble confidence reporting} (e.g., inclusion probabilities, IQR-based dispersion, intervention flip rates) and a principled evaluation protocol that integrates top-$k$ accuracy variants, keep/remove fidelity, sparsity, and stability.
    \1 We present comprehensive \emph{ablations} (perturbation type, surrogate nonlinearity, measurement regimes) and controlled evaluations on synthetic graphs with ground truth, demonstrating accurate and stable explanations, while clarifying trade-offs.
}

\section{Related Work}
Classical model-agnostic explainability methods enable insight into black-box models without accessing their internal parameters~\citep{xaidoshi}. Local approaches, such as LIME, approximate the behavior of a complex model around a specific input by fitting a simple surrogate model, usually a sparse-linear model, to perturbed samples, revealing which features drive an individual prediction~\citep{lime}. SHAP extends this idea via Kernel SHAP, which leverages Shapley values from cooperative game theory to assign each feature an additive importance score that satisfies consistency and accuracy properties~\citep{shap}. Global methods, in contrast, aim to capture overall model behavior: for example, partial dependence plots visualize how varying a feature influences predictions on average~\citep{pdp}, while global surrogate models approximate the entire decision surface~\citep{globalx}. Perturbation-based techniques underpin both LIME and SHAP by systematically altering input features and observing output changes, although they differ in how samples are weighted and how feature interactions are accounted for, to maintain fidelity to the original model.

Extending these approaches to graph data introduces additional challenges. In GNNs, node features are propagated and aggregated according to the graph's connectivity, meaning that predictions are determined not only by individual features but also by relational patterns across multiple nodes and edges~\citep{mpnn,gcn,gat}. Consequently, methods designed for tabular or image data cannot be directly applied to GNNs, as they fail to account for the combinatorial dependencies inherent in graph structures~\citep{xgnnrev}. Model-agnostic interpretability frameworks, like LIME, that address black-box explanations by fitting simple surrogate models locally around instances, do not factor in the graph structure, leading to explanations that may not reflect what drives the GNN output~\citep{glime}. To capture structural dependencies, methods focusing on subgraph-level explanations have emerged. GNNExplainer learns masks on edges and node features by maximizing mutual information between masked and original predictions, but it relies on non-convex optimization without global optimality guarantees~\citep{gnne}. PGExplainer uses a parameterized neural network to generate explanations, improving generalization across multiple graphs~\citep{pge}. SubgraphX searches for important connected subgraphs using Monte Carlo Tree Search and employs Shapley values to measure subgraph contributions, offering interpretable explanations at the cost of more complex computations~\citep{subx}. Building on these insights, convex statistical frameworks have been introduced to enhance both stability and interpretability in graph explanations. GraphLIME extends LIME to graphs by sampling the local N-hop neighborhood and using HSIC Lasso for nonlinear feature selection~\citep{glime,hsic}. This approach ensures convex optimization and yields stable attributions of node features that best explain local GNN behavior~\citep{glime}. StGraphLIME~\citep{stglime} builds on this by introducing Group~\citep{groupl} and Fused Lasso~\citep{fusedl} regularization to select coherent subgraphs as explanations. It perturbs input graphs and applies HSIC-based selection on these perturbations, identifying substructures critical for GNN predictions. These developments mark a shift from simple feature attribution to interpretable, structure-aware explanation methods, that reveal subgraphs or clusters guiding GNN decisions, enhancing both fidelity and interpretability.

Explaining quantum neural networks (QNNs) carries additional complexity due to the probabilistic nature of quantum measurements and high-dimensional Hilbert spaces~\citep{qnnrev}. Classical explainability tools cannot be applied directly because QNN model outputs are probabilistic and noisy. Q-LIME~\citep{qlime} extends LIME to quantum settings by perturbing inputs and fitting an ensemble of local surrogate models to measurement outcomes for each input. This ensemble-based approach helps delineate regions where randomness from quantum measurements dominates, offering insight into model behavior under stochastic regimes. Other techniques include SVQX, which applies Shapley values over quantum feature subsets, and hybrid classical-quantum explainers that approximate quantum circuits with differentiable classical surrogates for interpretability~\citep{qshap,powerxaiqnn}. Explaining QNN face broader challenges such as the inability to account for phase information, entanglement, and non-commutative feature representations, as highlighted by \citet{shapig}.

While extensive research has explored explainability for classical machine learning models~\citep{xaidoshi} as well as graph neural networks~\citep{gnnrev}, to the best of our knowledge, no methods currently exist for explaining the predictions of Quantum Graph Neural Networks. Existing quantum explainability techniques focus on standalone quantum neural circuits or hybrid models, but do not address the relational structure of graph data processed by QGNNs. This gap highlights the need for dedicated approaches that can account for both quantum coherence and graph topology in generating meaningful explanations. In this work, we take a first step toward addressing this need by proposing an explanation framework tailored specifically to QGNNs.

\section{Background}
\subsection{Graph Neural Networks (GNNs)}

Let \(\cG = (\cV, \cE, X)\) denote an undirected graph with node set \(\cV = \{v_1, \dots, v_n\}\), edge set \(\cE \subseteq \cV \times \cV\), and node feature matrix \(X = [\mathbf{x}_v] \in \R^{n \times d}\), where \(\mathbf{x}_v \in \R^d\) is the feature vector for node \(v\). A GNN~\citep{mpnn} computes node representations by iterative message passing and update operations. Denote by \(\mathbf{h}_v^{(l)} \in \R^h\) the hidden state of node \(v\) at layer \(l\), with initial state \(\mathbf{h}_v^{(0)} = \mathbf{x}_v\). The message passing step aggregates information from the neighbors \(\mathcal{N}(v) = \{u : (u,v) \in \cE\} \cup \{v\} \):
\[
\mathbf{m}_v^{(l)} 
= \sum_{u \in \mathcal{N}(v)} M^{(l)}\bigl(\mathbf{h}_v^{(l-1)},\,\mathbf{h}_u^{(l-1)},\,\mathbf{x}_{uv}\bigr),
\]
where, \(M^{(l)}: \R^h \times \R^h \times \R^e \to \R^h\) is a learnable message function and \(\mathbf{x}_{uv} \in \R^e\) denotes optional edge features. The update step then computes the new node state:
\[
\mathbf{h}_v^{(l)} 
= U^{(l)}\bigl(\mathbf{h}_v^{(l-1)},\,\mathbf{m}_v^{(l)}\bigr),
\]
where, \(U^{(l)}: \R^h \times \R^h \to \R^h\) is a learnable update function. For graph-level tasks, a readout function \(R\) aggregates the final node states \(\{\mathbf{h}_v^{(L)}\}\) into a single graph representation:
\[
\mathbf{h}_{\cG} \;=\; R\bigl(\{\mathbf{h}_v^{(L)} : v \in \cV\}\bigr),
\]
which is then passed to a task-specific classifier or regressor.

\subsection{Structural GraphLIME (StGraphLIME)}

Structural GraphLIME~\citep{stglime} extends local feature explanations to the graph structure by identifying coherent subgraphs that drive the prediction of a GNN, instead of merely selecting individual features. By leveraging kernel-based (in)dependence measures and structural sparsity penalties, StGraphLIME highlights connected groups of nodes and edges, whose joint variation most strongly influences the model's output. The following paragraphs review the mathematical tools that make this approach effective.

\paragraph{Hilbert-Schmidt Independence Criterion (HSIC)}

Given a set of perturbed feature-output pairs $\{(\mathbf{x}_i,\mathbf{y}_i)\}_{i=1}^p$, choose positive-definite kernels $k:\R^p\times\R^p\to\R$ and $\ell:\R^p\times\R^p\to\R$. Define the raw Gram matrices as:
\[
    K^{\mathrm{raw}}_{ij}=k(\mathbf{x}_i,\mathbf{x}_j),
    \quad
    L^{\mathrm{raw}}_{ij}=\ell(\mathbf{y}_i,\mathbf{y}_j),
\]
and the centering matrix as:
\[
    H=I_n-\frac{1}{n}\mathbf{1}_n\mathbf{1}_n^\top.
\]
The normalized centered Gram matrices are
\[
    K = \frac{H\,K^{\mathrm{raw}}\,H}{\|H\,K^{\mathrm{raw}}\,H\|_F},
    \quad
    L = \frac{H\,L^{\mathrm{raw}}\,H}{\|H\,L^{\mathrm{raw}}\,H\|_F},
\]
where, $\|\cdot\|_F$ denotes the Frobenius norm. The normalized HSIC is then
\[
    \mathrm{HSIC}(\mathbf{x},\mathbf{y})
    =
    \mathrm{tr}(K\,L).
\]
For Gaussian kernels,
\[
    k(\mathbf{x}_i,\mathbf{x}_j)
    =
    \exp\!\biggl(-\frac{\|\mathbf{x}_i-\mathbf{x}_j\|_2^2}{2\sigma_x^2}\biggr),
    \quad
    \ell(\mathbf{y}_i,\mathbf{y}_j)
    =
    \exp\!\biggl(-\frac{\|\mathbf{y}_i-\mathbf{y}_j\|_2^2}{2\sigma_y^2}\biggr),
\]
where, $\sigma_x>0$ and $\sigma_y>0$ are the kernel bandwidths.

\paragraph{HSIC Lasso (HSIC-L1) for individual node/edge selection.}
Let $\{\hat{\cG}_i\}_{i=1}^p$ be perturbed graphs of $\cG$, and let $L\in\R^{p\times p}$ be the normalized Gram matrix of predictions $\Phi(\hat{\cG}_i)$. For each binary mask over nodes/edges, $\nu \in \cA$, let $K_\nu\in\R^{p\times p}$ be the normalized Gram matrix computed from perturbations affecting $\nu$. Then, the HSIC Lasso objective for selecting nodes or edges is
\eq{
    \min_{\mathbf{\alpha} \in \R^{|\cA|}}
    \biggl\{
        \frac{1}{2}\bigl\lVert L - \sum_{\nu\in\cA}\alpha_\nu\,K_\nu\bigr\rVert_F^2
        \;+\;
        \lambda\sum_{\nu\in\mathcal{A}}|\alpha_\nu|
    \biggr\},
    \quad \alpha_\nu \ge 0,
    \label{eq:l1}
}
where, $\lambda > 0$ is a sparsity parameter, and $\alpha_\nu$ indicates the importance score, where $\cA$ can be $\cV$ or $\cE$.

\paragraph{Group-Regularized HSIC Lasso (HSIC-G).}
Let $\Pi$ be a collection of overlapping groups, $\pi \subseteq \cA$. For each $\pi\in\Pi$, define $\mathbf{\gamma}_\pi \coloneqq (\alpha_v)_{v\in\pi}$. The group-regularized HSIC lasso is
\eq{
    \min_{\mathbf{\alpha}  \in \R^{|\cA|}}
    \biggl\{
        \frac{1}{2}\bigl\lVert L - \sum_{\nu\in\cA}\alpha_\nu\,K_\nu\bigr\rVert_F^2
        \;+\;
        \lambda\sum_{\pi\in\Pi}\|\mathbf{\gamma}_\pi\|_2
    \biggr\},
    \quad \alpha_\nu\ge0.
    \label{eq:grp}
}
Here, the group penalty term $|\boldsymbol{\gamma}_\pi|_2$ encourages the selection of entire groups $\pi$, while $\mathcal{A}$, as before, denotes either $\mathcal{V}$ or $\mathcal{E}$, thereby promoting structural coherence. Here, $\pi$ represents typically connected subsets of nodes or edges that collectively influence the model's prediction. For nodes, groups are defined by 1-hop neighborhoods, whereas for edges, grouping is based on shared incident nodes.

\subsection{Equivariant Quantum Graph Circuits (EQGCs)}

EQGCs represent a specialized class of parameterized quantum circuits designed to process graph-structured data~\citep{eqgc}. The circuits are designed to respect the symmetry of such data by enforcing the equivariance condition with respect to node permutations. Formally, EQGCs take as input a graph represented by its adjacency matrix \(A \in \{0,1\}^{n \times n}\) and node features \(\{x_v\}_{v=1}^n\). The graph features are encoded into a quantum state \(|v\rangle = |v_1\rangle \otimes \cdots \otimes |v_n\rangle\), where each \(|v_i\rangle\) encodes the feature vector \(x_{v_i}\) on a local Hilbert space associated with node \(i\). The circuit defines a unitary operator \(C_\theta(A)\), parameterized by \(\theta\), and structured depending on the graph topology:
\[
    C_\theta(P^\top A P) = \tilde{P}^\top C_\theta(A) \tilde{P},
\]
where, \(P\) is a node permutation matrix and \(\tilde{P}\) is the corresponding unitary acting on the total Hilbert space, ensuring equivariance. Two important classes of EQGCs exist: Equivariant Hamiltonian QGCs (EH-QGCs) and Equivariantly Diagonalizable Unitary QGCs (EDU-QGCs). EH-QGCs encode symmetry directly through Hamiltonian-based edge and node operators, while EDU-QGCs achieve equivariance via diagonalizable commuting two-node unitaries, parameterized by local rotations. In this work, we employ EDU-QGCs due to their simpler and more tractable implementation. The model output is obtained by measuring in the computational basis, yielding outcome probabilities that depend on both the quantum parameters and graph structure, suitable for graph learning tasks, like classification.

\subsection{Quantum LIME (Q-LIME)}

Let $f_\theta$ denote the quantum model to be explained, $x \in \mathbb{R}^d$ the instance of interest, and $p$ the number of perturbations. Q-LIME~\citep{qlime} constructs local datasets $\{D_i^x\}_{i=1}^m$ within a neighborhood $\pi_x$ of the feature space centered at $x$, each comprising $p$ perturbed instances generated through systematic perturbations of $x$. The model $f_\theta$ is then queried to obtain corresponding labels $f_\theta(D_i^x)$. To quantify the variability of explanations arising from quantum stochasticity, Q-LIME defines a set $\Xi \coloneqq \{S_i\}_{i=1}^m$ of $m$ classical LIME surrogates. Each surrogate $S_i$ is fit on $(D_i^x, f_\theta(D_i^x))$, thereby yielding a local decision boundary $B_i$. The collection of these boundaries, $B = \{B_1, \dots, B_m\}$, is used to characterize uncertainty in the explanation space. Notably, Q-LIME defines a local region of indecision $R$ for a data point $\mathbf{x}' \in X$ as
\[
    R = \left\{ \mathbf{x}' \in X : \left| P\bigl(g(\mathbf{x}') = 1 \mid f_\theta, \Xi, \pi_{\mathbf{x}}\bigr) - \frac{1}{2} \right| < \varepsilon \right\},
\]
where, $g(\mathbf{x}')$ denotes the surrogate model prediction, $\varepsilon > 0$ is a small threshold hyperparameter, and the probability is computed over the surrogate ensemble $\Xi$. Specifically, Q-LIME picks the inter-quartile range $\mathrm{IQR}$ as the summary statistic to quantify the uncertainty. Sorting $B$ yields the first and third quartiles, $Q_1$ and $Q_3$, respectively, and their difference
\[
    \mathrm{IQR}(B) = Q_3 - Q_1
\]
captures the dispersion of decision boundaries arising from quantum measurement noise. A larger $\mathrm{IQR}$ indicates higher uncertainty and reduced confidence in the corresponding explanations.

\section{Quantum GraphLIME}
\label{sec:method}

\begin{figure}[t]
    \centering
    \includegraphics[width=1.0\textwidth]{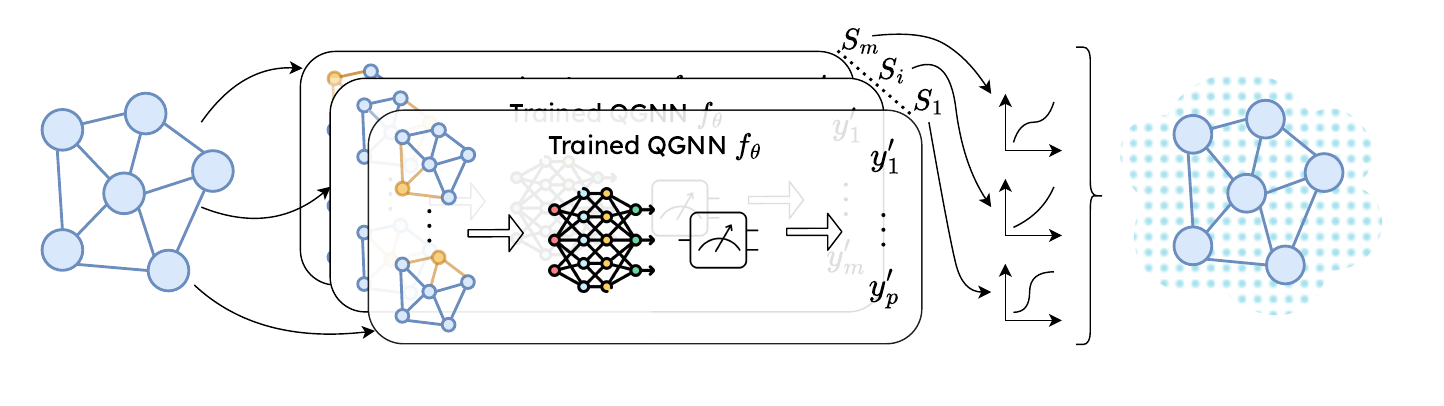}
    \caption{QGraphLIME Workflow: Our method generates locally perturbed graph datasets ($\{D_i\}_{i=1}^m$) from an input graph and evaluates them with a trained QGNN ($f_\theta$) to capture quantum stochasticity. An ensemble of surrogate models $\Xi = \{S_i\}_{i=1}^m$ is fit on $(D_i, f_\theta(D_i))$, with boundary dispersion used to quantify explanation uncertainty. Aggregated surrogate attributions rank influential nodes or edges, yielding stable, uncertainty-aware explanations; see Algorithm~\ref{alg:qgraphlime}.}
    \label{fig:diagram}
\end{figure}

\begin{algorithm}[h]
    \caption{Quantum GraphLIME (QGraphLIME) for Node/Edge Importance Ranking}\label{alg:qgraphlime}
    \begin{algorithmic}[1]
        \STATE \bfa{Inputs:} Trained QGNN $f_{\theta}$, input graph $G=(V,E)$, number of perturbations $m$, measurement shots $s$, number of surrogate models $n$.
        \STATE \bfa{Initialize:} Surrogate ensemble $\Xi = \{S_i\}_{i=1}^n$.
        \STATE \bfa{Step 1:} Generate $m$ locally perturbed graph datasets $\{D_i\}_{i=1}^m$ from $G$; encode perturbations in a binary matrix $Z \in \{0,1\}^{m \times |\cA|}$, $\cA \in \{\cV, \cE\}$.
        \STATE \bfa{Step 2:} Evaluate each $D_i$ $s$ times with $f_\theta$ to capture stochastic predictions.
        \STATE \bfa{Step 3:} Fit each surrogate $S_i$ on $(D_i, f_\theta(D_i))$ to approximate local decision boundaries $B_i$.
        \STATE \bfa{Step 4:} Compute dispersion of $B = \{B_1, \dots, B_m\}$ via graph-structure-aware summary statistics to quantify uncertainty.
        \STATE \textbf{Step 5:} Aggregate surrogate attributions to rank influential nodes/edges and produce stable, uncertainty-aware explanations.
        \STATE \textbf{Outputs:} Surrogate ensemble $\Xi$, node/edge importance scores, summary statistics of uncertainty.
    \end{algorithmic}
\end{algorithm}

In this section, we introduce \bfa{Quantum GraphLIME} (\bfa{QGraphLIME}), a model-agnostic framework that provides principled, post-training explanations for Quantum Graph Neural Networks by explicitly accounting for measurement-induced stochasticity. Given a trained QGNN $f_{\theta}$ and an input graph $G = (V, E)$, QGraphLIME constructs multiple locally perturbed graph datasets $\{D_i\}_{i=1}^m$ from $G$, following the locality-driven perturbation strategy of LIME. For graphs, perturbations are generated via random node removals and random-walk-based modifications to preserve local structure. Each perturbation removes $r$ nodes, producing $p$ graph variants represented as rows of a binary encoding matrix $Z$. To capture quantum randomness, each perturbed graph is evaluated $s$ times, corresponding to the number of quantum measurement shots, generating stochastic graph-level predictions. An ensemble of interpretable surrogate models $\Xi = \{S_i\}_{i=1}^m$, inspired by Q-LIME, is then fit on the pairs $(D_i, f_\theta(D_i))$, approximating local decision boundaries $B_i$. The dispersion among these boundaries, quantified through graph-structure-aware summary statistics, reflects uncertainty in the explanations.

For interpretability, QGraphLIME employs HSIC-based surrogates, specifically the HSIC-L1 Lasso~(\eqc{l1}) and HSIC Group Lasso~(\eqc{grp}), which leverage the Hilbert-Schmidt Independence Criterion to capture nonlinear dependencies between graph features and model outputs. This choice follows the design philosophy of GraphLIME~\citep{glime} and StGraphLIME~\citep{stglime}, offering enhanced sensitivity to graph structural complexity. We illustrate the full QGraphLIME pipeline in \picc{diagram} and detail each step in \algc{qgraphlime}.

To quantify the ensemble-level importance and uncertainty of graph elements (nodes or edges), we compute summary statistics over the surrogate attributions. Let \(\mathbf{s}^{(n)} \in \mathbb{R}^N\) denote the vector of importance scores assigned by the \(n\)-th surrogate in an ensemble of \(m\) surrogates to the \(N\) elements of a graph. We define the following metrics to aggregate and interpret surrogate outputs:

\paragraph{Top-\(k\) Inclusion Probability (TIP)} 
The Top-\(k\) Inclusion Probability measures the proportion of surrogate models in which an element appears among the top-\(k\) most important elements. Formally, for element \(i \in \{1, \dots, N\}\):
\[
\mathrm{TIP}_i = \frac{1}{m} \sum_{n=1}^{m} \mathbb{I}\Big[ i \in \mathrm{Top}\text{-}k\big(\mathbf{s}^{(n)}\big) \Big],
\]
where, \(\mathbb{I}[\cdot]\) is the indicator function and \(\mathrm{Top}\text{-}k(\mathbf{s}^{(n)})\) returns the indices of the \(k\) elements with the highest scores in surrogate \(n\). Higher TIP values indicate stronger consensus among surrogate models regarding the importance of that element.

\paragraph{Interquartile Range (IQR)} 
To assess variability in element importance across the surrogate ensemble, we compute the interquartile range for each element:
\[
\mathrm{IQR}_i = Q_3\left(\{s_i^{(n)}\}_{n=1}^m\right) - Q_1\left(\{s_i^{(n)}\}_{n=1}^m\right),
\]
where, \(Q_1\) and \(Q_3\) denote the first and third quartiles, respectively. Larger IQR values indicate greater disagreement among surrogates, reflecting higher uncertainty in the assigned importance. We additionally compute a $90\%$ confidence interval for each element by considering the distribution of scores across the surrogate ensemble.

\paragraph{Flip Probabilities under Element Removal} 
To provide a causal, intervention-based measure of importance, we define the flip probability for element \(i\) as the fraction of trials in which the model's predicted label changes when that element is removed. Let \(f_\theta(\cG)\) denote the QGNN's prediction for graph \(\cG\), and let \(\cG_{\setminus i}\) represent the graph with element \(i\) removed. Over \(T\) repeated evaluations capturing quantum measurement stochasticity, the flip probability is
\[
\mathrm{Flip}_i = \frac{1}{T} \sum_{t=1}^{T} \mathbb{I} \big[ f_\theta(\cG_{\setminus i}^{(t)}) \neq f_\theta(\cG) \big].
\]
This metric complements TIP and IQR by assessing the causal influence of individual graph elements on the model's predictions, providing a rigorous, intervention-based evaluation of explanatory fidelity.

Building on these ensemble-level metrics, we evaluate the contribution of each design choice to QGraphLIME's performance through comprehensive ablation studies. Specifically, we examine the effects of the type of graph perturbations (random vs. random-walk), and the choice of surrogate models (HSIC-L1 vs. HSIC Group Lasso). We also compare against a linear Logistic Regression surrogate to assess the effect of surrogate nonlinearity on the QGLIME's performance. For each configuration, we measure explanatory validity using the metrics described above. In addition, we assess explanatory validity through causal keep/remove interventions, quantified by $\mathrm{Fidelity}^{+}$ and $\mathrm{Fidelity}^{-}$, and evaluate the \emph{sparsity} of the resulting explanations. This analysis allows us to isolate the impact of each component on the stability, accuracy, and interpretability of the resulting node or edge rankings.

In the next subsection, we provide a theoretical analysis of QGraphLIME, establishing bounds on the number of surrogate models needed to reliably capture quantum measurement stochasticity while maintaining practical efficiency.

\subsection{Theoretical Analysis}

A key consideration in QGraphLIME is choosing the number of surrogate models to reliably capture the inherent measurement stochasticity of QGNNs, while maintaining practical efficiency. Unlike classical neural networks, whose outputs are deterministic, quantum measurements introduce unavoidable randomness that must be sufficiently sampled to produce faithful explanations. Each surrogate model represents a sample from the distribution of quantum measurement outcomes (i.e., predicted class probabilities for a graph), and the ensemble forms an empirical approximation of this distribution. From a practical standpoint, limiting the number of surrogates is important for performance, particularly on NISQ hardware due to low qubit counts and noise, and in classical statevector simulations due to the cost at high qubit counts and many circuit evaluations, which constrain computational resources. We leverage the Dvoretzky-Kiefer-Wolfowitz (DKW) inequality~\citep{dakineq, massart1990} to establish a principled, distribution-free bound on the minimum surrogate count required to approximate the quantum measurement distribution within a desired accuracy ($\epsilon$) and confidence ($\delta$), providing a practical criterion for selecting the ensemble size that balances explanatory fidelity with computational efficiency.

\paragraph{Assumptions.}
We fix a graph, \(\cG\), and a scalar explanation statistic, \(\mathcal{T}(S)\), computed from an independently generated surrogate \(S\) (e.g., a class probability for \(\cG\) or an aggregated node/edge attribution), under a fixed ensemble-generation recipe, that resamples perturbations and measurement seeds and fits the surrogate with fixed hyperparameters. We assume \(\{\mathcal{T}(S_i)\}_{i=1}^m\) are i.i.d.\ draws from a distribution with CDF \(\cP\). This models each surrogate as a sample from the stochastic pipeline induced by quantum measurements.

\begin{theorem}[Minimum surrogate count]
\label{thm:main}
Let \(\cP\) denote the true CDF of the scalar statistic \(\mathcal{T}(S)\) induced by the ensemble-generation randomness, and let \(\hat{\cP}_m\) be the empirical CDF from \(m\) i.i.d.\ surrogates \(S_1, \dots, S_m\). For any \(\epsilon > 0\) and \(\delta\in(0, 1)\), if
\[
m \;\ge\; \frac{1}{2\,\epsilon^2}\,\ln\!\Big(\frac{2}{\delta}\Big),
\]
then, with probability at least \(1 - \delta\),
\[
\sup\limits_{t\in\R} \big|\hat{\cP}_m(t) - \cP(t)\big| \;\le\; \epsilon,
\]
so, the surrogate ensemble captures the distribution of \(\mathcal{T}(S)\) within tolerance \(\epsilon\) at confidence \(1 - \delta\).
\end{theorem}
\begin{proof}
By the DKW inequality with Massart's sharp constant, \(\Pr\!\big(\sup\limits_{t} |\hat{\cP}_m(t) - \cP(t)| > \epsilon\big) \le 2 e^{-2 m \epsilon^2}\). Setting \(2 e^{-2 m \epsilon^2} = \delta\) and solving for \(m\) yields the stated bound, which is distribution-free and holds for all \(\epsilon>0\) \citep{dakineq,massart1990}.
\end{proof}

\begin{corollary}[Simultaneous guarantee across graphs/statistics]
For \(n\) graphs and \(K\) scalar statistics per graph (e.g., class-wise scores or node/edge attribution aggregates), requiring uniform \(\epsilon\)-accuracy simultaneously for all \(nK\) CDFs is achieved if
\[
    m \;\ge\; \frac{1}{2\,\epsilon^2}\,\ln\!\Big(\frac{2 n K}{\delta}\Big),
\]
by Boole's inequality (union bound) applied to the \(nK\) DKW events.
\end{corollary}

\paragraph{Discussion.}
In our setting, the statistic \(\mathcal{T}(S)\) is the binary class probability in \([0,1]\), so the supremum in the DKW deviation can be taken over \([0,1]\). When \(\mathcal{T}\) is continuous, this follows from the probability integral transform, which maps \(\mathcal{T}\) to \(\mathrm{Uniform}[0, 1]\) and makes suprema over \([0,1]\) and \(\R\) equivalent, which provides a convenient formulation for probability-valued outputs. Even when the probability estimate is effectively discrete, e.g., due to a finite number of measurement shots or discrete perturbations, the DKW inequality applies verbatim with the supremum over \(\R\) and the same sharp two-sided constant, so the stated bound on \(m\) remains valid. The guarantee presumes that the surrogate-level statistics \(\{\mathcal{T}(S_i)\}_{i=1}^m\) are i.i.d.\ under a fixed ensemble-generation protocol (fixed circuit, resampled perturbations and measurement seeds, fixed training hyperparameters), ensuring that \(\hat{P}_m\) is the empirical CDF of i.i.d.\ draws. In the presence of correlated noise, hardware drift, or shared randomness across surrogates, the effective sample size can be reduced, making the bound conservative in practice while preserving its distribution-free form for independent approximations or after accounting for dependence.

\section{Experiments}

\subsection{Datasets}

To facilitate rigorous evaluation of explanation quality, we employ synthetically generated graph datasets with clearly defined ground-truth targets. Graphs are drawn from canonical structures, including wheels, cycles, two-connected wheels, two-connected cycles, and heterogeneous combinations, with controlled variations relevant to the classification tasks. To prevent overfitting to fixed node identities, hub node positions are randomized through relabeling, ensuring that neither the model nor the explainer can exploit a fixed node index. Each graph contains designated reference nodes, such as hub nodes in wheel graphs, that serve as ground-truth targets and visual anchors for qualitative assessment.

In all experiments, we train an EQGC model for binary graph classification and apply Quantum GraphLIME to explain its predictions. We focus on two primary tasks, following the experimental setup of StGraphLIME~\citep{stglime}:

\begin{itemize}
    \item \textbf{Case 1:} Distinguishing cycles from wheels by identifying hub nodes.
    \item \textbf{Case 2:} Detecting the presence of disconnected hub nodes in two-connected graphs.
\end{itemize}

The datasets are structured as follows. Case 1 comprises 100 training graphs (50 wheels and 50 cycles) and 40 test wheel graphs, while Case 2 includes 180 training graphs (120 two-wheel and 60 two-cycle) and 40 test two-wheel graphs. Graph sizes are constrained by the maximum number of qubits that classical statevector simulations on our experimental setup can accommodate. Specifically, Case 1 graphs contain up to 13 nodes and Case 2 graphs up to 16 nodes. These controlled datasets provide sufficient structural variability while preserving interpretable targets, enabling quantitative and qualitative evaluation of explanations. Our code is available at \href{https://github.com/smlab-niser/qglime}{https://github.com/smlab-niser/qglime}.

\subsection{Model Training}

We instantiate the EQGC model with two quantum layers, allocating one qubit per node, and employ a hidden-layer MLP with 32 units to aggregate node-level embeddings. The model produces a single class probability as a graph-level output, and $2000$ measurement shots are used during training to capture quantum stochasticity. We implement the model using PyTorch\footnote{\url{https://pytorch.org/}} and PyTorch Geometric\footnote{\url{https://pytorch-geometric.readthedocs.io/en/2.6.1/}}, using the original code as reference\footnote{\url{https://github.com/pmernyei/eqgc-experiments/}}, while ensuring that the quantum measurement noise is faithfully simulated. Supervised training over 50 epochs for each case is conducted using binary cross-entropy loss between predicted and true labels, optimized via Adam~\citep{adam}. Node-level quantum probabilities are classically aggregated to generate graph-level predictions. Validation is performed throughout training to monitor convergence and prevent overfitting, enabling robust generalization to unseen graphs.

\subsection{Explanation Methodology}

QGraphLIME is applied post hoc to elucidate QGNN predictions. Following the locality-driven perturbation strategy of LIME, each input graph is perturbed through random node removals and random-walk-based modifications. Perturbed instances are evaluated $s$ times using the trained QGNN to capture measurement-induced stochasticity. HSIC-based surrogate models, including HSIC-L1 Lasso and HSIC Group Lasso, are fit on these perturbed graphs to approximate local decision boundaries. An ensemble of surrogate models enables aggregation of the resulting importance scores, with Top-\(k\) Inclusion Probability (TIP) capturing consensus across surrogates, interquartile range (IQR) quantifying uncertainty, and flip probabilities providing a causal, intervention-based assessment of influence, thereby offering a comprehensive measure of confidence for node- or edge-level explanations.

In addition to these per-input-graph metrics, which summarize surrogate-level variability and uncertainty, we assess global explanation performance by aggregating results across the dataset. This includes standard measures of fidelity, sparsity, and confidence, enabling a holistic evaluation of QGraphLIME's ability to provide accurate, stable, and interpretable explanations across diverse graph instances.

\subsection{Evaluation Metrics}

We assess explanatory performance using multiple complementary metrics that are standard in classical GNN explanation literature~\citep{gnnexp}.

\paragraph{Top-k Accuracy} This metric quantifies the fraction of test instances for which the ground-truth target nodes are included among the top-$k$ nodes ranked by importance, by the surrogate model
\[
\mathrm{Acc}_{\mathrm{top-}k} = \frac{1}{N} \sum_{i=1}^N \mathbb{I}\left[\mathrm{target}_i \in \mathrm{Top}\text{-}k(\mathbf{\alpha}_i)\right],
\]
where, $\alpha_i$ denotes the node importance scores for instance $i$, and $\mathbb{I}[\cdot]$ is the indicator function. Higher values indicate that the explainer accurately identifies the most influential nodes. For Case 1 (single-target tasks), \emph{One@1} and \emph{One@3} accuracy measure the percentage of instances where the ground-truth target node appears among the highest-ranked or top three nodes by importance, respectively. In Case 2 (dual-target tasks), \emph{Both@2} and \emph{Both@6} indicate the proportion of samples where both target nodes are included within the top 2 or top 6 ranked nodes, while \emph{One@2} and \emph{One@6} reflect the fraction of cases where at least one of the two target nodes is present in these respective subsets. These variants offer a more nuanced evaluation of the explainer's ability to reliably highlight all relevant nodes or at least partially recover important targets in settings with multiple correct answers.

\paragraph{Fidelity} Fidelity measures the extent to which the surrogate explanation aligns with the QGNN's predictions under targeted interventions. In particular, \textit{Fidelity-plus} ($\fidp$) evaluates agreement when only the top-$k$ nodes are retained and all other nodes are removed, whereas \textit{Fidelity-minus} ($\fidm$) evaluates agreement when the top-$k$ nodes are removed and all remaining nodes are kept. A good explainer is characterized by a \emph{high} $\fidp$ and a \emph{low} $\fidm$, indicating that it correctly identifies nodes that are critical to the model's decision-making.

\paragraph{Sparsity} This metric assesses the conciseness of the explanation by computing the fraction of nodes with negligible importance:
\[
    S = 1 - \frac{|\{i : s_i \ge 0.1 \max_j s_j\}|}{N}.
\]
Higher sparsity reflects explanations that focus on a smaller subset of influential nodes, enhancing interpretability.

\paragraph{Confidence Metrics} 
These metrics evaluate the stability and distinctiveness of explanations across multiple surrogate instances. Let $\mathbf{s}^{(n)} \in \mathbb{R}^N$ denote the importance scores assigned by surrogate model $n$ to the $N$ nodes of a graph, and let $\mathcal{T} \subseteq \{1,\dots,N\}$ be the set of target nodes with ground-truth importance. Define the mean score of node $v$ across the ensemble as 
\[
\mu_v = \frac{1}{m} \sum_{n=1}^{m} s_v^{(n)},
\]
where, $m$ is the number of surrogate models.

\begin{itemize}
    \item \textbf{Consensus} ($\cons$) measures the fraction of surrogate models in which each target node appears among the top-$k$ nodes. Formally, for node $v \in \mathcal{T}$:
    \[
        \cons_v = \frac{1}{m} \sum_{n=1}^{m} \mathbb{I}\Big[v \in \mathrm{Top}\text{-}k(\mathbf{s}^{(n)})\Big],
    \]
    and the overall consensus across all target nodes is 
    \[
        \cons = \frac{1}{|\mathcal{T}|} \sum_{v \in \mathcal{T}} \cons_v.
    \]
    Higher values indicate stronger agreement among surrogates regarding the importance of target nodes.

    \item \textbf{Relative Importance} ($\ri$) quantifies how strongly target nodes are distinguished from non-target nodes. Let $\mathcal{N} = \{1,\dots,N\} \setminus \mathcal{T}$ denote the set of non-target nodes. Then,
    \[
        \ri = \frac{\frac{1}{|\mathcal{T}|} \sum_{v \in \mathcal{T}} \mu_v}{\frac{1}{|\mathcal{N}|} \sum_{u \in \mathcal{N}} \mu_u},
    \]
    Larger values indicate that target nodes are clearly more important than non-target nodes across the surrogate ensemble.
\end{itemize}

\noindent Collectively, these ensemble and per-graph metrics provide a comprehensive assessment of explanations, capturing accuracy, interpretability, and robustness across the dataset, while explicitly accounting for variability induced by the measurement stochasticity inherent in QGNNs.

\subsection{Ablation Studies}
To rigorously assess the contribution of each design choice, we perform comprehensive ablation studies. We vary the perturbation type (random vs. random-walk), surrogate model choice (HSIC-L1 vs. HSIC Group Lasso), and surrogate linearity (linear logistic regression vs. nonlinear HSIC-based models). Additionally, we examine the effect of repeated measurements, comparing StGraphLIME (single-shot) against QGraphLIME (multi-shot). For each configuration, explanatory validity is evaluated using Top-k Accuracy, Fidelity, and Sparsity metrics, allowing us to isolate the impact of each component on stability, accuracy, and interpretability of the resulting node or edge rankings. \tabc{ablations} presents a table of all evaluated design combinations, summarizing the perturbation strategies, surrogate model variants, and measurement regimes considered in our ablation study, while \secc{ablations} provides a detailed analysis of the results, highlighting the individual and combined effects of these design choices.

\subsection{Implementation Details}
All experiments were conducted on a system with a 24-core AMD CPU and an NVIDIA A100 GPU with 80~GiB VRAM. Graph perturbations, surrogate model fitting, and ensemble aggregation were implemented in Python using standard scientific libraries. Hyperparameters, including the number of perturbations, measurement shots, and surrogate regularization coefficients, were selected based on preliminary cross-validation to balance fidelity, sparsity, and computational efficiency.

\section{Results and Discussion}

We evaluate QuantumGraphLIME using previously discussed standard graph explanation metrics, as well as ensemble metrics, which together allow us to assess both the precision and stability of explanations for quantum graph neural networks.

\begin{figure}[ht]
    \centering
    \begin{subfigure}{0.30\textwidth}
        \includegraphics[width=\linewidth]{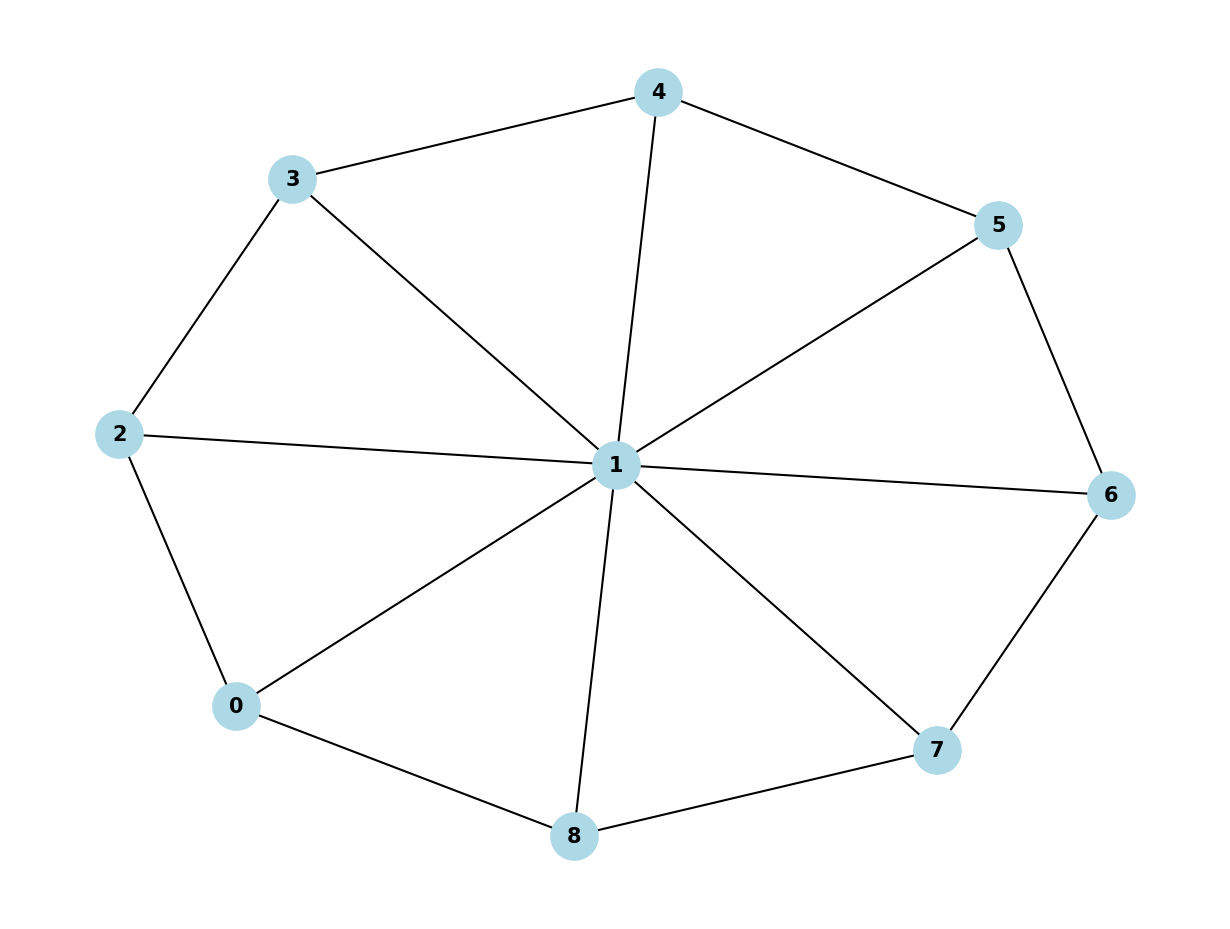}
    \end{subfigure}%
    \begin{subfigure}{0.33\textwidth}
        \includegraphics[width=\linewidth]{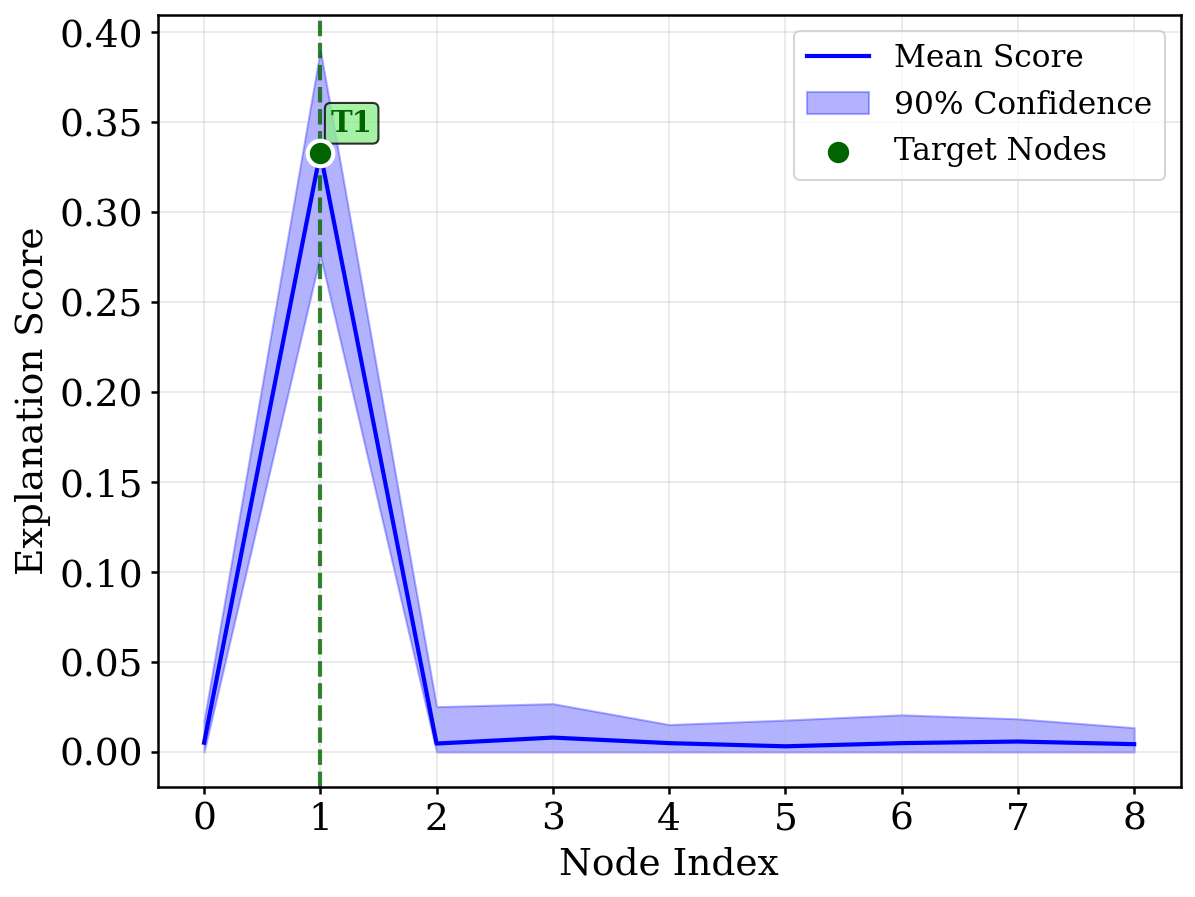}
    \end{subfigure}%
    \begin{subfigure}{0.33\textwidth}
        \includegraphics[width=\linewidth]{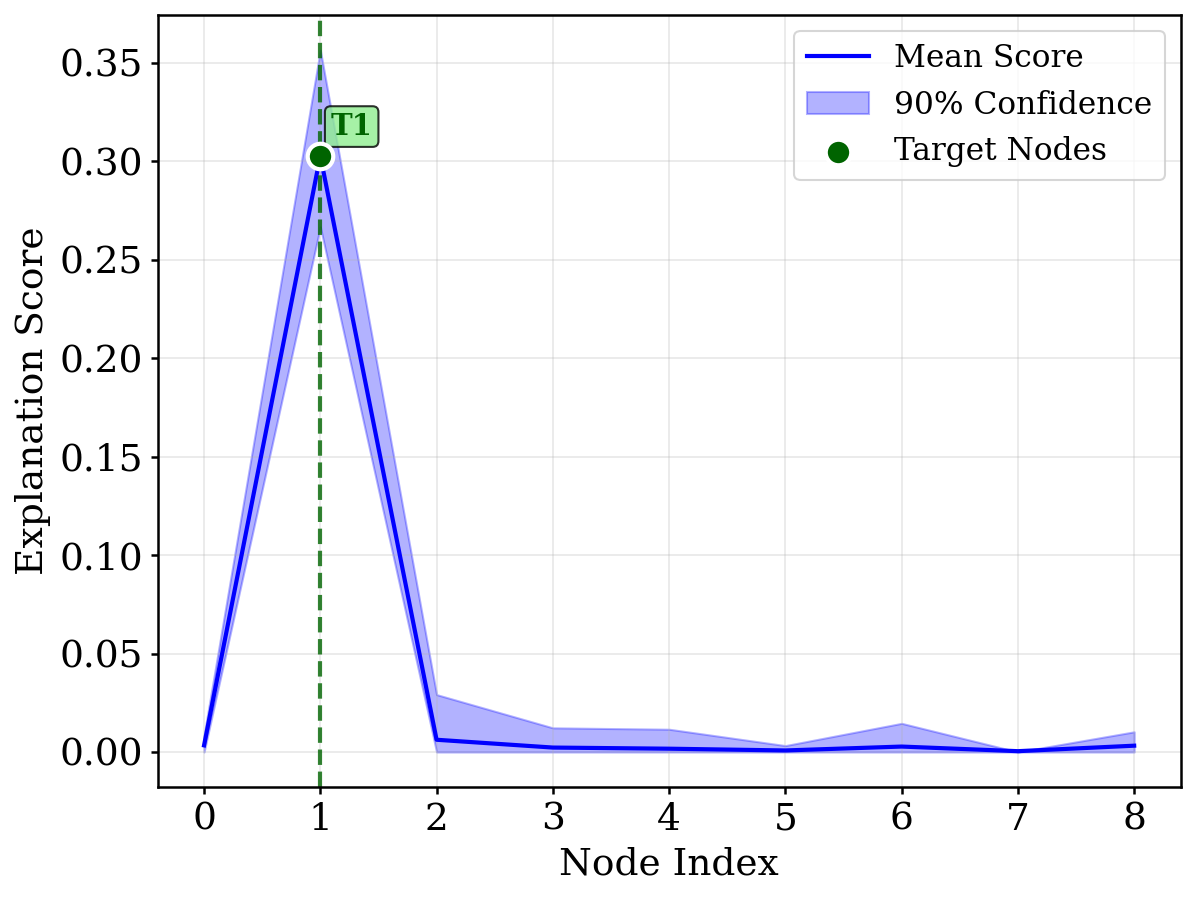}
    \end{subfigure}

    \vspace{0.4cm}

    \begin{subfigure}{0.30\textwidth}
        \includegraphics[width=\linewidth]{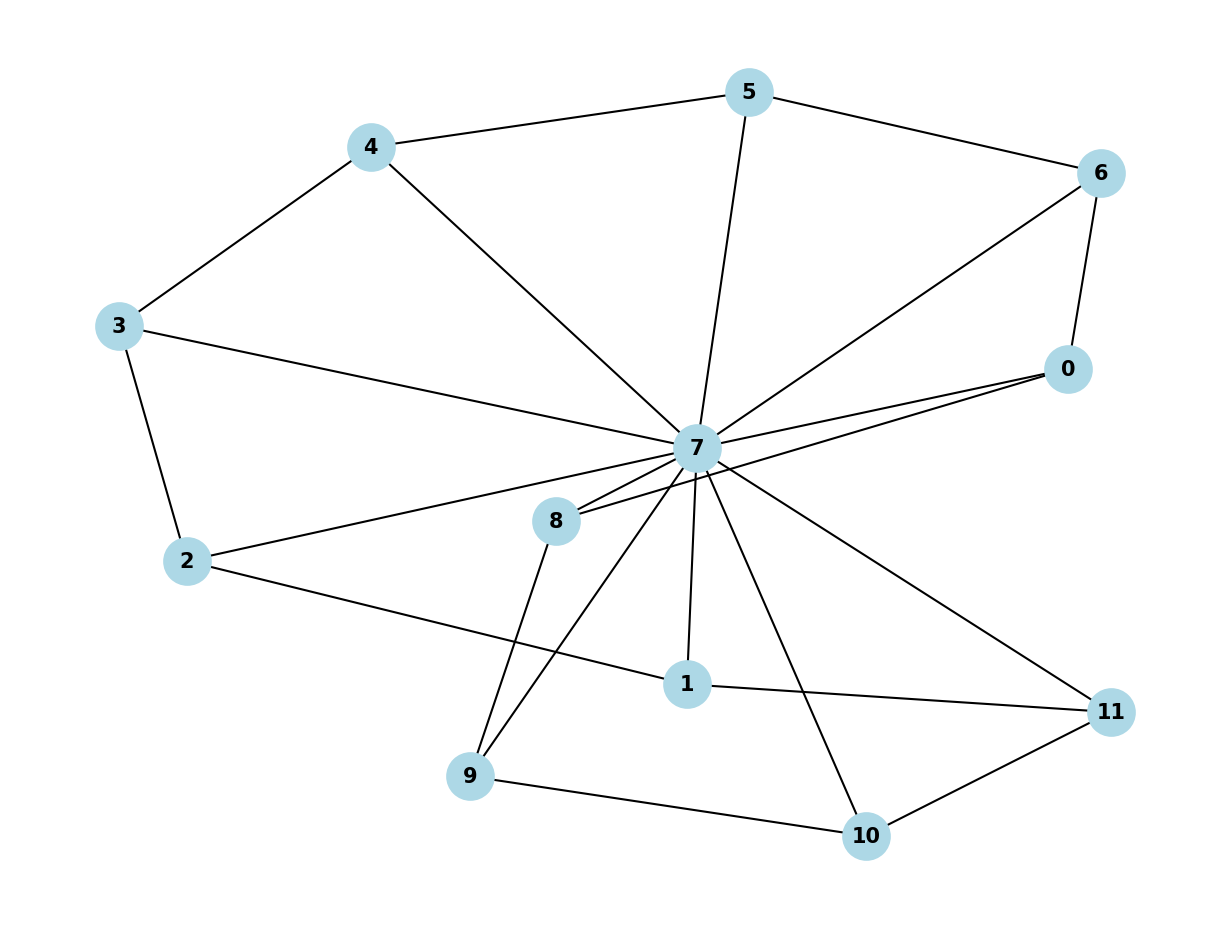}
    \end{subfigure}%
    \begin{subfigure}{0.33\textwidth}
        \includegraphics[width=\linewidth]{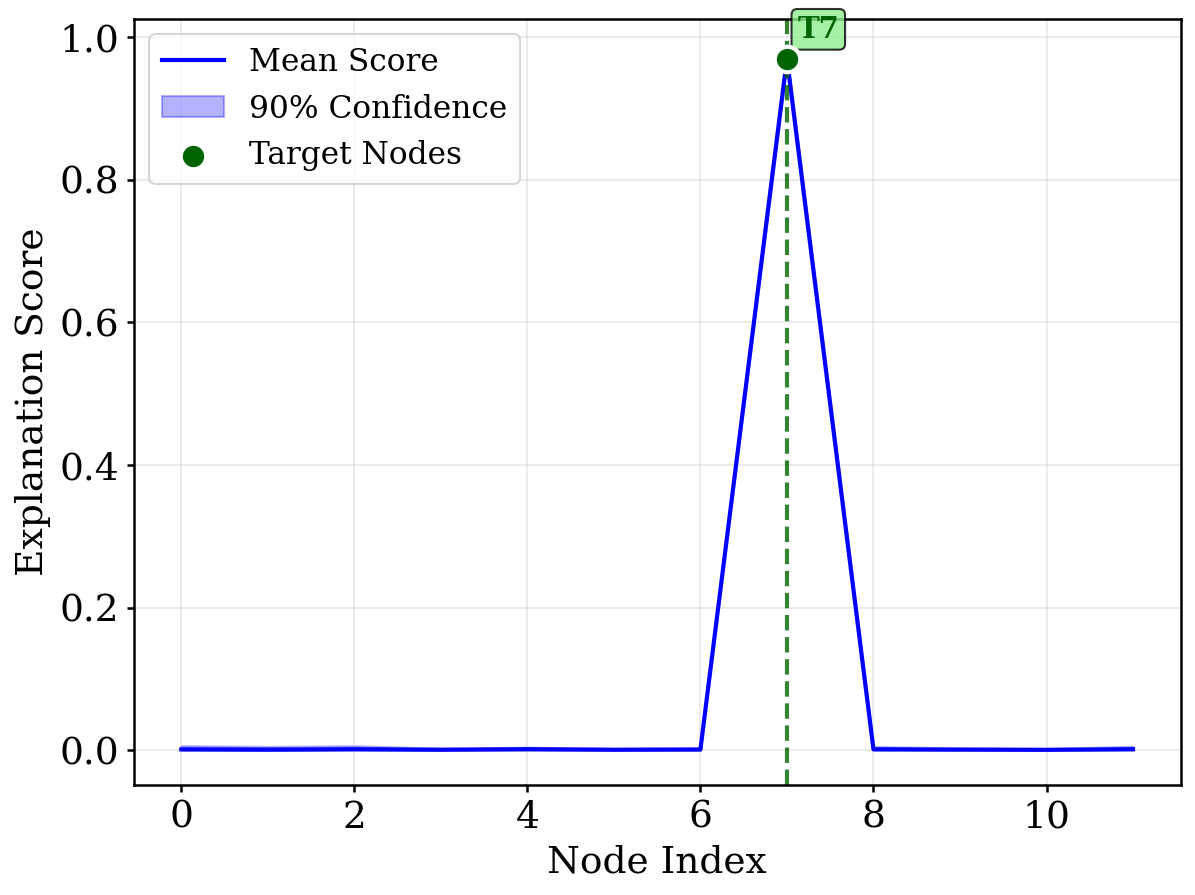}
    \end{subfigure}%
    \begin{subfigure}{0.33\textwidth}
        \includegraphics[width=\linewidth]{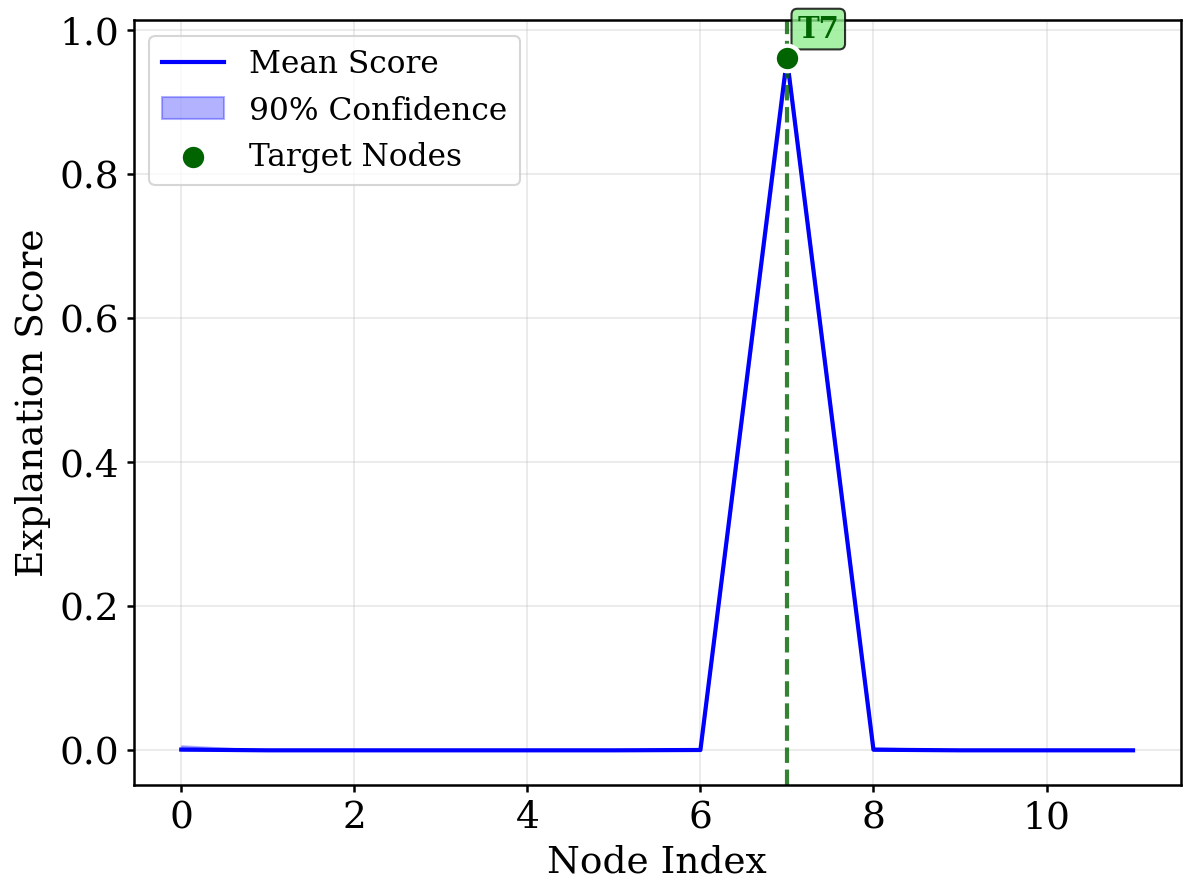}
    \end{subfigure}

    \caption{
        Case 1 - Single-Target: Explanation Score variation across surrogates per node due to quantum stochasticity. \bfa{Left} - Input Graph; \bfa{Center} - QGLIME-HSIC L1; \bfa{Right} - HSIC-G.
    }
    \label{fig:hsicsc1}
\end{figure}

\begin{figure}[ht]
    \centering
    \begin{subfigure}{0.30\textwidth}
        \includegraphics[width=\linewidth]{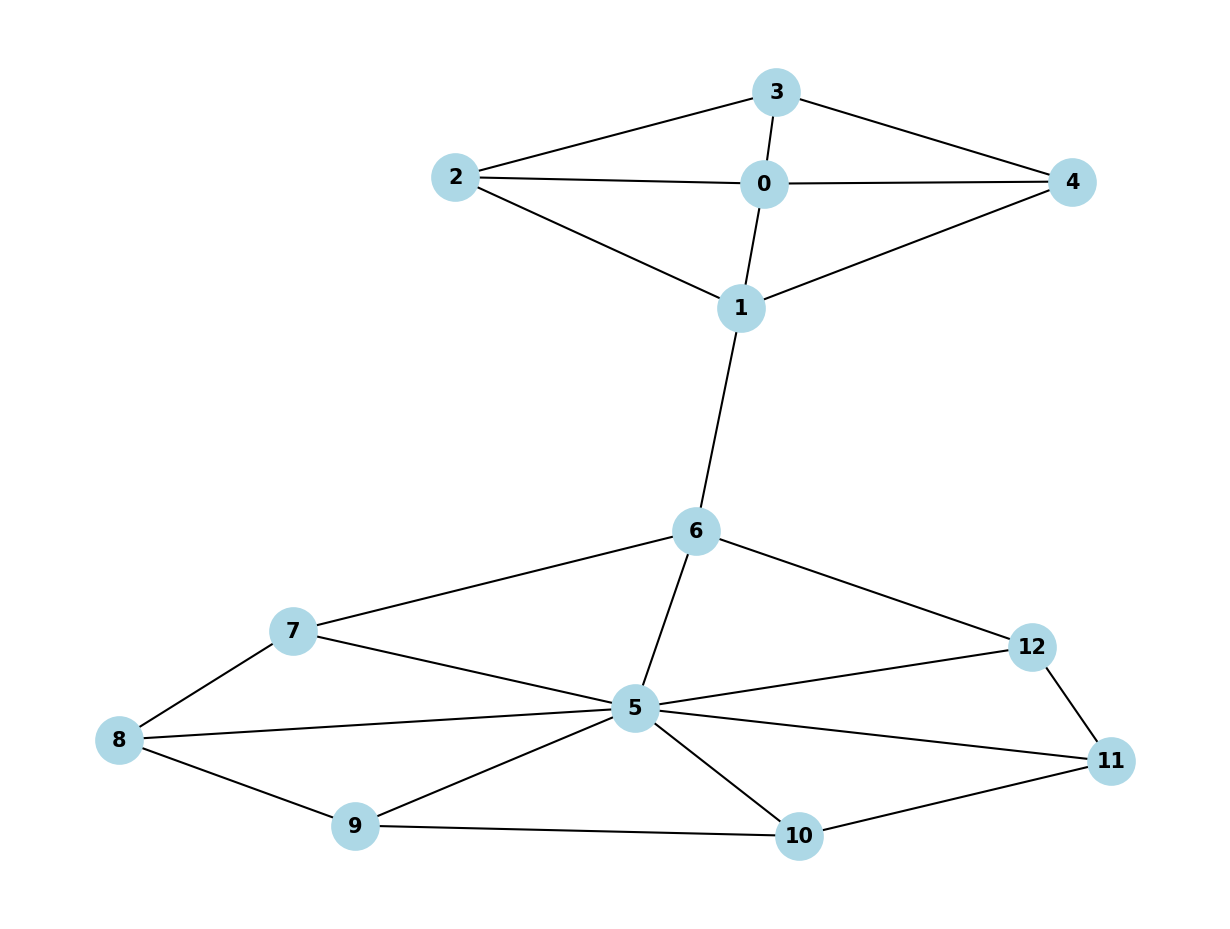}
    \end{subfigure}%
    \begin{subfigure}{0.33\textwidth}
        \includegraphics[width=\linewidth]{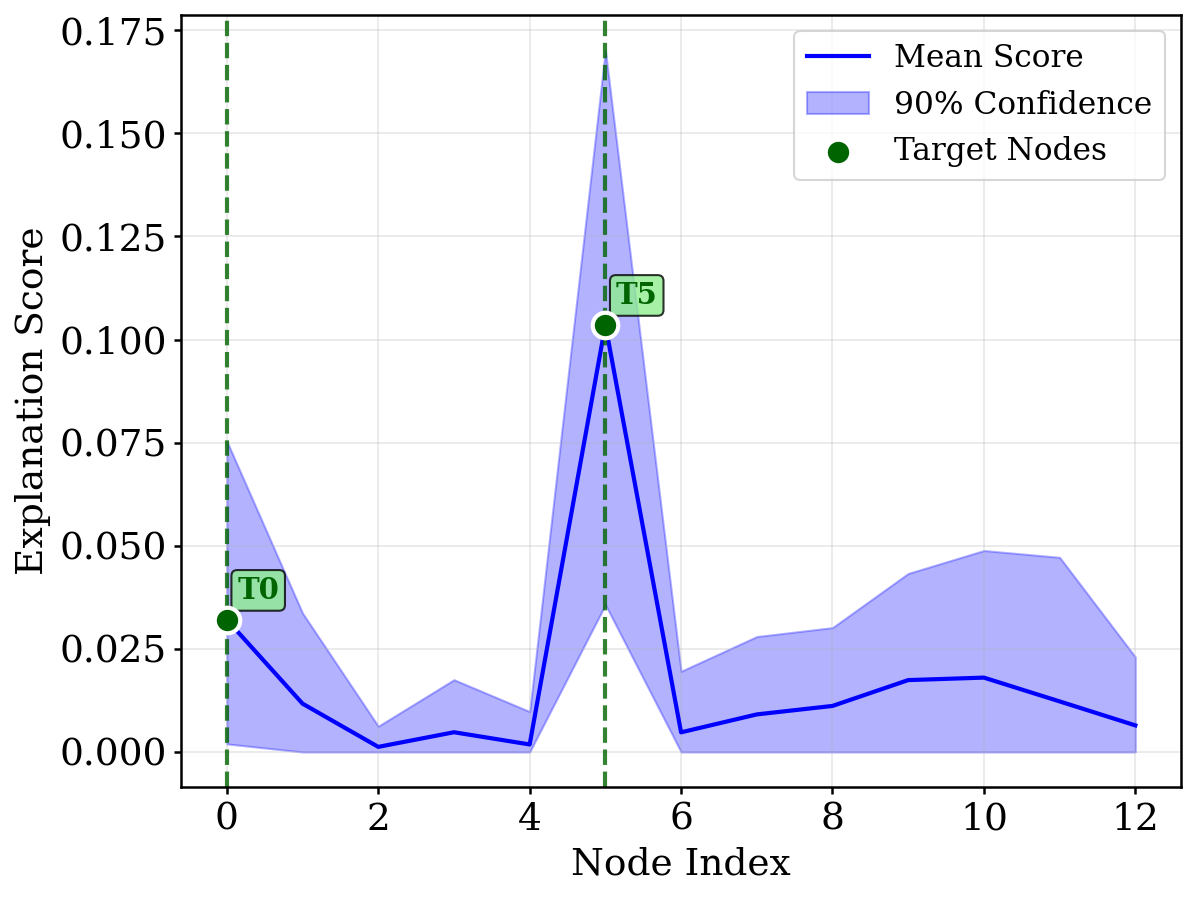}
    \end{subfigure}%
    \begin{subfigure}{0.33\textwidth}
        \includegraphics[width=\linewidth]{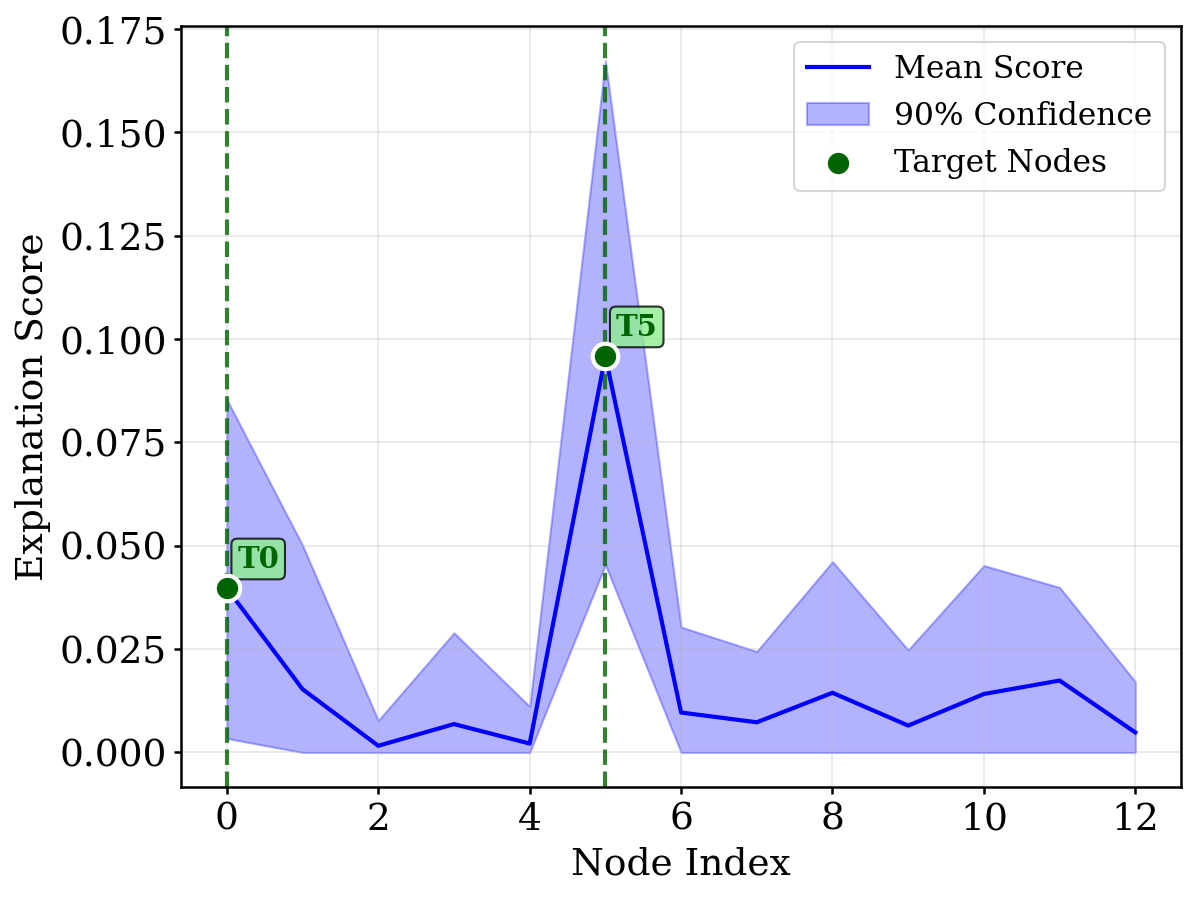}
    \end{subfigure}

    \vspace{0.4cm}

    \begin{subfigure}{0.30\textwidth}
        \includegraphics[width=\linewidth]{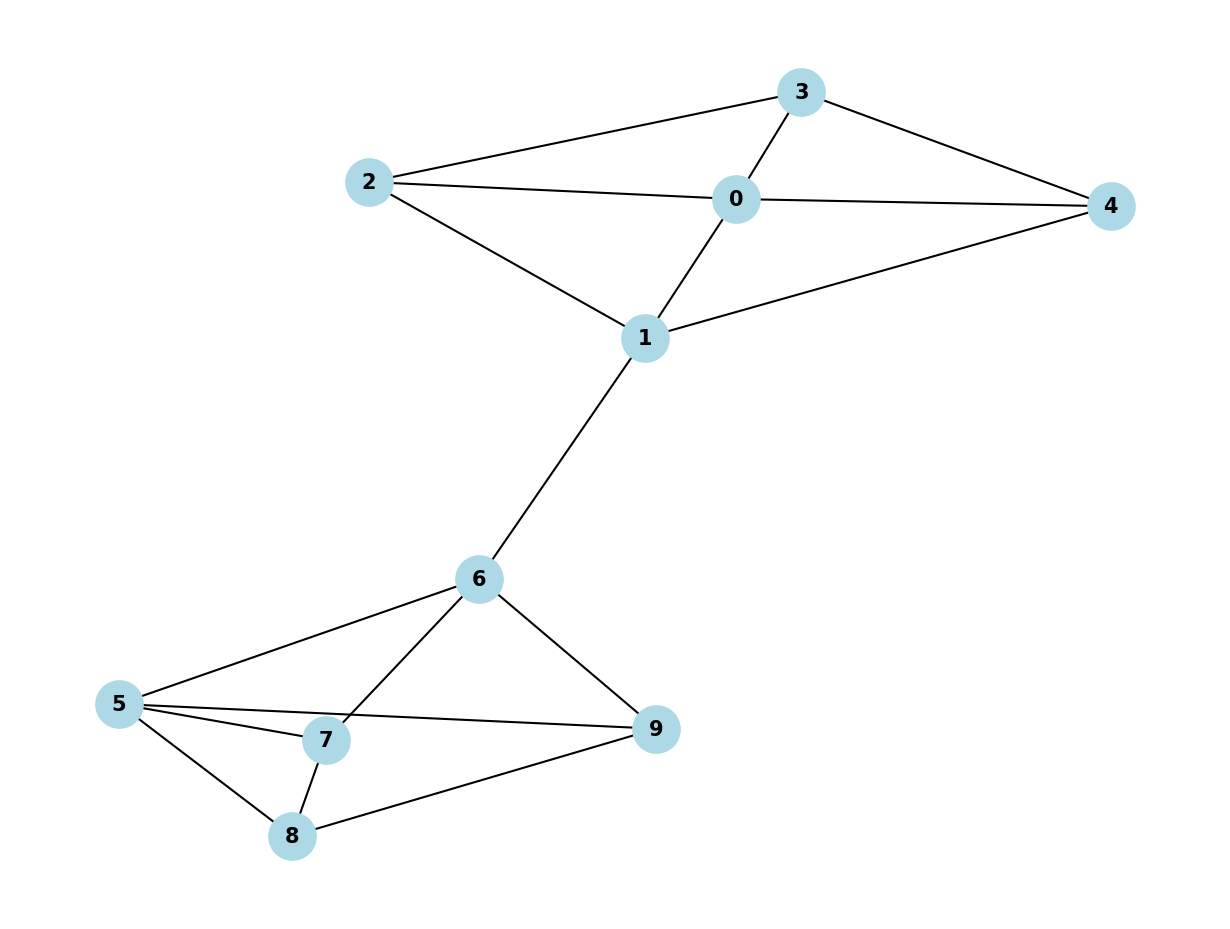}
    \end{subfigure}%
    \begin{subfigure}{0.33\textwidth}
        \includegraphics[width=\linewidth]{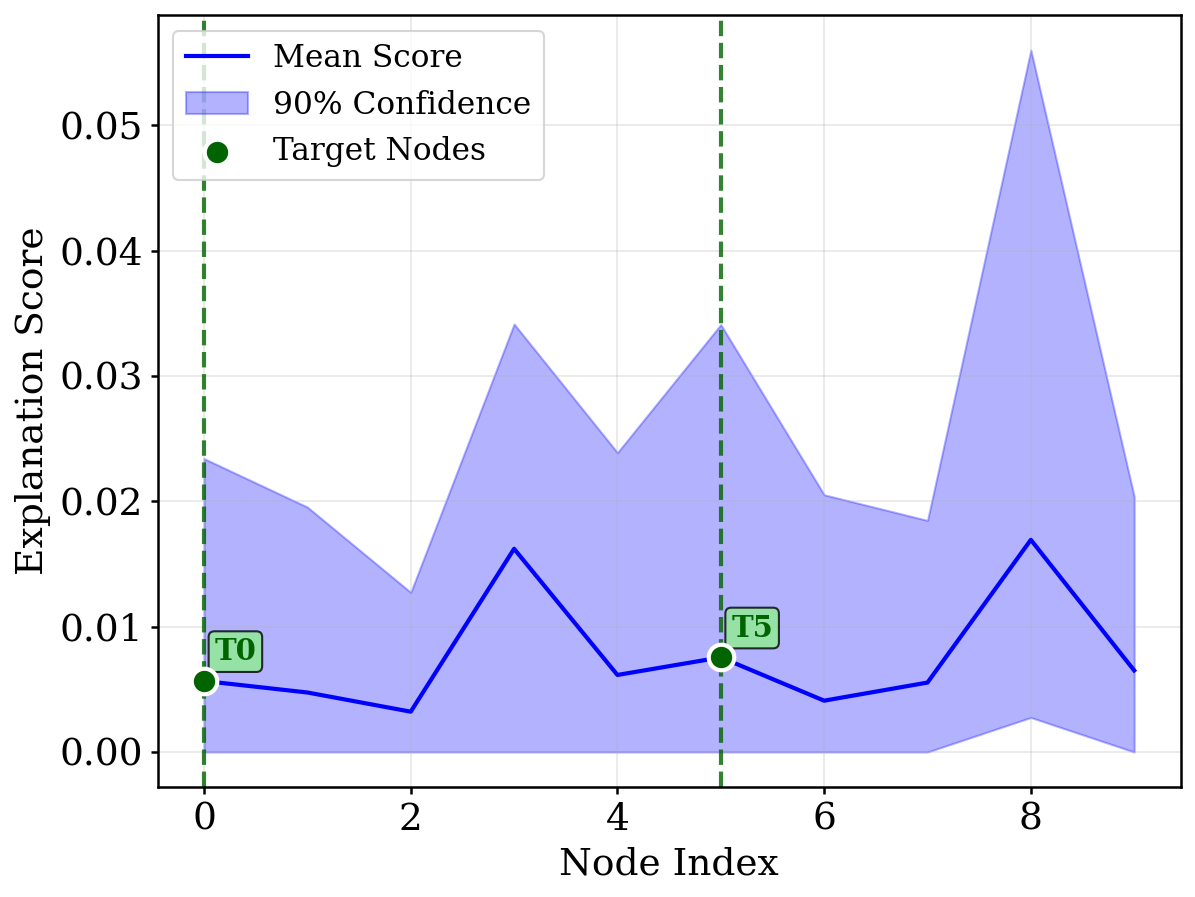}
    \end{subfigure}%
    \begin{subfigure}{0.33\textwidth}
        \includegraphics[width=\linewidth]{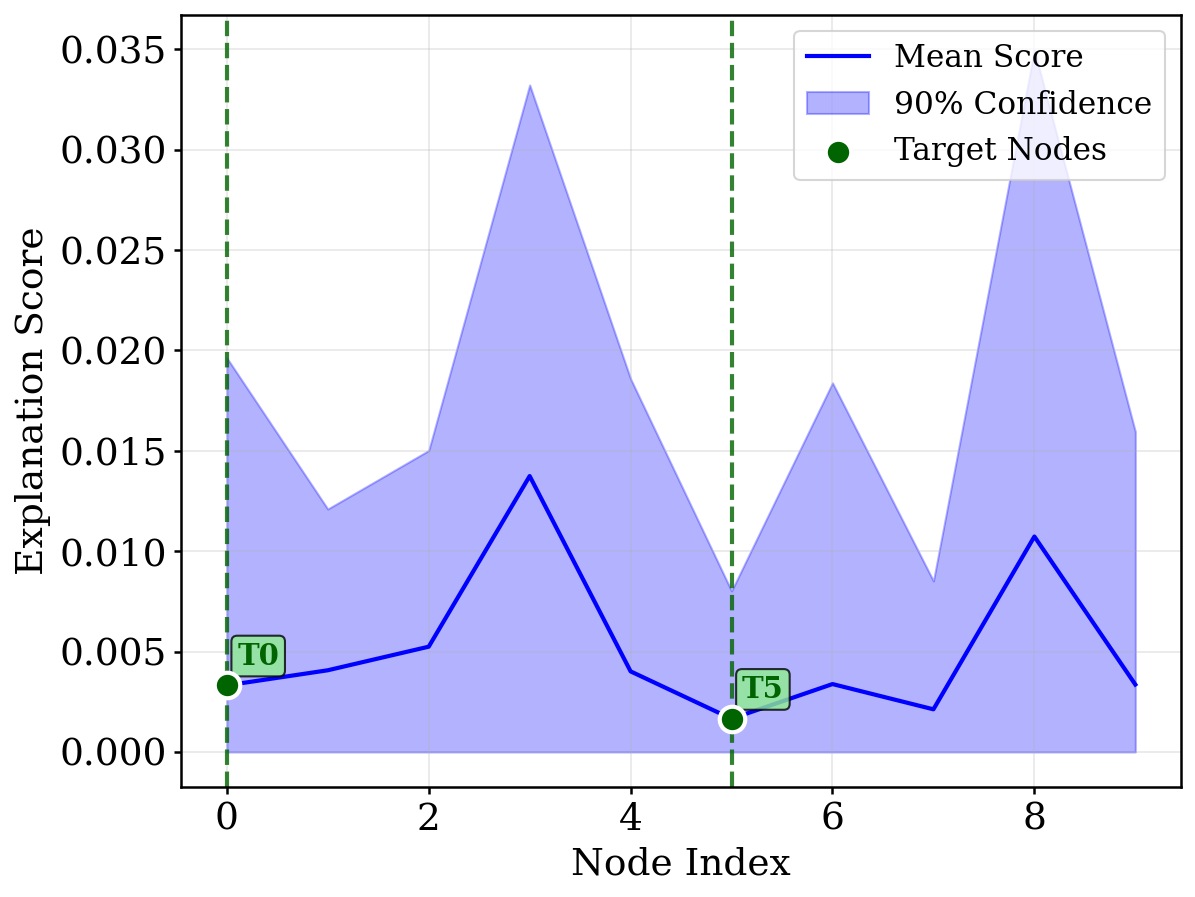}
    \end{subfigure}

    \caption{
        Case 2 - Dual-Target: Explanation Score variation across surrogates per node due to quantum stochasticity. \bfa{Left} - Input Graph; \bfa{Center} - QGLIME-HSIC L1; \bfa{Right} - HSIC-G.
    }
    \label{fig:hsicsc2}
\end{figure}

\begin{table}[ht]
    \centering
    \caption{Case 1 and Case 2: Explanation Accuracy for QGLIME variants vs. a Random baseline}
    \begin{tabular}{lcc|cccc}
        \toprule
        Method & \emph{One@1} & \emph{One@3} & \emph{Both@2} & \emph{Both@6} & \emph{One@2} & \emph{One@6} \\
        \midrule
        Random Explainer & $0.10 \pm 0.30$ & $0.30 \pm 0.46$ & $0.1 \pm 0.30$ & $0.3 \pm 0.46$ & $\sbest{0.4 \pm 0.49}$ & $0.8 \pm 0.40$ \\
        QGLIME (HSIC-L1) & $\best{1.00 \pm 0.00}$ & $\best{1.00 \pm 0.00}$ & $\best{0.9 \pm 0.30}$ & $\best{1.0 \pm 0.00}$ & $\best{0.9 \pm 0.30}$ & $\best{1.0 \pm 0.00}$ \\
        QGLIME (HSIC-G) & $\best{1.00 \pm 0.00}$ & $\best{1.00 \pm 0.00}$ & $\best{0.9 \pm 0.30}$ & $\sbest{0.9 \pm 0.30}$ & $\best{0.9 \pm 0.30}$ & $\sbest{0.9 \pm 0.30}$ \\
        \bottomrule
    \end{tabular}
    \label{tab:mainacc}
\end{table}

\begin{table}[ht]
\centering
\caption{Case 1 and Case 2: QGLIME Explainability Performance}
\begin{tabular}{lccccc}
    \toprule
    Method & $\fidm$ & $\fidp$ & $\spars$ & $\cons$ & $\ri$ \\
    \midrule
    \multicolumn{6}{l}{\textbf{Case 1: Single-Target}} \\
    \addlinespace[0.3em]
    QGLIME (HSIC-L1) & $\sbest{0.963 \pm 0.030}$ & $\sbest{0.844 \pm 0.256}$ & $\sbest{0.726 \pm 0.363}$ & $\sbest{1.0}$ & $\sbest{7.690}$ \\
    QGLIME (HSIC-G)  & $\best{0.975 \pm 0.030}$  & $\best{0.866 \pm 0.265}$  & $\best{0.746 \pm 0.327}$  & $\best{1.0}$ & $\best{8.456}$ \\
    \addlinespace[0.2em]
    \midrule
    \multicolumn{6}{l}{\textbf{Case 2: Dual-Target}} \\
    \addlinespace[0.3em]
    QGLIME (HSIC-L1) & $\best{0.001 \pm 0.000}$ & $\sbest{0.863 \pm 0.290}$ & $\best{0.544 \pm 0.231}$ & $\best{0.240}$ & $\best{0.834}$ \\
    QGLIME (HSIC-G)  & $\best{0.001 \pm 0.000}$  & $\best{0.873 \pm 0.285}$  & $\sbest{0.467 \pm 0.253}$ & $\sbest{0.180}$ & $\sbest{0.426}$ \\
    \bottomrule
\end{tabular}
\label{tab:mainmetric}
\end{table}

\subsection{QGLIME Performance Analysis}

Tables~\ref{tab:mainacc} and \ref{tab:mainmetric}, as well as Figures~\ref{fig:hsicsc1} and \ref{fig:hsicsc2}, summarize the evaluation of QGLIME across all test cases and surrogate configurations, providing empirical evidence for the patterns in accuracy, fidelity, sparsity, and confidence that we discuss below.

\paragraph{Single-Target (Hub) Detection.}  
In the single-target setting, QGLIME consistently produces highly reliable explanations. Across surrogates, consensus and relative importance metrics indicate that the framework robustly identifies the critical hub node in each graph instance, while sparsity metrics confirm that the explanations remain focused and concise. Fidelity measures further show that the surrogate models closely approximate the QGNN's response to interventions, suggesting that the explanations are not only accurate but also reflective of the underlying quantum decision process. Collectively, these results demonstrate that QGLIME effectively balances interpretability, fidelity, and reproducibility for simple graph structures.

\paragraph{Dual-Target (Hubs) Detection.}  
When extending to multiple targets, QGLIME continues to maintain high explanatory accuracy; however, metrics capturing ensemble uncertainty and reproducibility reveal emerging limitations. Consensus decreases and IQRs increase, indicating that surrogate models show greater variability in assigning importance when multiple nodes jointly influence predictions. The decline in sparsity reflects the need to distribute importance across a larger subset of nodes to preserve explanation validity. Notably, fidelity exhibits asymmetry under top-k retention versus top-k removal interventions, suggesting that interactions between multiple influential nodes introduce complex dependencies that are harder for local surrogates to capture. These observations illustrate that while QGLIME remains effective in multi-target scenarios, explanation confidence is naturally constrained by combinatorial complexity.

\paragraph{Discussion.} The comparison between single and dual-target scenarios highlights several key insights. First, ensemble-based surrogates provide valuable uncertainty quantification, revealing not only which nodes are influential, but also how stable these assessments are under perturbation and stochastic quantum measurements. Second, surrogate choice and perturbation strategy have a measurable impact on explanation stability, particularly in more complex multi-target settings. Finally, the observed trade-offs among sparsity, fidelity, and confidence indicate that achieving high accuracy for interacting targets may require less sparse explanations. Together, these findings demonstrate that QGLIME provides an uncertainty-aware framework for explaining QGNNs, effectively capturing the relational complexity inherent in graph-structured data as well as the intrinsic stochasticity of quantum measurements.

\subsection{Ablation Studies}
\label{sec:ablations}

To evaluate the contribution of each component, we conduct systematic ablation studies on QGraphLIME. \tabc{ablations} summarizes the variations considered in these experiments, which collectively allow us to disentangle the effects of perturbation, surrogate formulation, and quantum measurement regimes on overall explanatory quality. Specifically, we compare random node perturbations with random-walk-based perturbations to assess the sensitivity of explanations to structural locality. We further analyze the impact of surrogate model linearity by contrasting logistic regression with nonlinear HSIC-based surrogates (HSIC-L1 and HSIC Group Lasso), thereby isolating the effect of nonlinear dependencies on explanation fidelity. Further, the influence of the regularization strength ($\lambda$) within the HSIC framework is examined to determine its role in balancing sparsity and fidelity. Finally, we compare single-shot (StGraphLIME) and multi-shot (QGraphLIME) configurations to evaluate the benefits of repeated quantum measurements in stabilizing explanations under inherent stochasticity. Each configuration is evaluated using the previously defined metrics: Top-(k) Accuracy, Fidelity ($\fidp$ and ($\fidm$), and Sparsity, along with ensemble-based confidence metrics to quantify stability and interpretability.
\begin{table}[ht]
    \centering
    \caption{Ablation Studies: QuantumGraphLIME Components}
    \begin{tabular}{ll}
        \toprule
        \textbf{Component} & \textbf{Variants / Combinations} \\
        \midrule
        Perturbation Type & Random Node vs. Random Walk \\
        Surrogate Linearity & Linear (Logistic Regression) vs. Nonlinear (HSIC-based) \\
        Quantum Measurement & StGraphLIME (Single-shot) vs. QGraphLIME (Multi-shot) \\
        HSIC-Group Penalty, \(\lambda\) & \(\{1, 10^{-1}, 10^{-2}, 10^{-3}, 10^{-4}\}\) \\
        \bottomrule
    \end{tabular}
    \label{tab:ablations}
\end{table}

\subsubsection{Effect of Different Perturbation Strategies: Random Walk}

\begin{figure}[ht]
    \centering

    \begin{subfigure}{0.49\textwidth}
        \includegraphics[width=\linewidth]{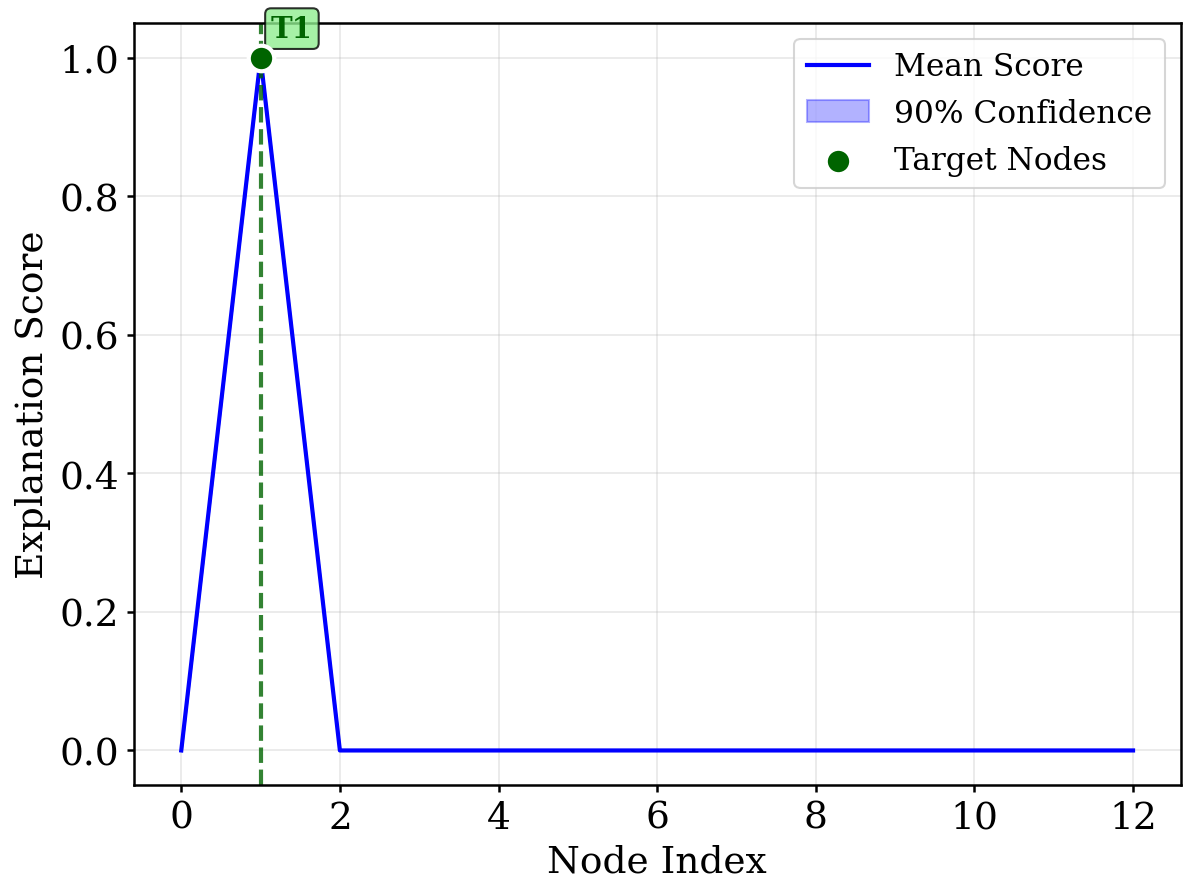}
    \end{subfigure}%
    \hfill
    \begin{subfigure}{0.49\textwidth}
        \includegraphics[width=\linewidth]{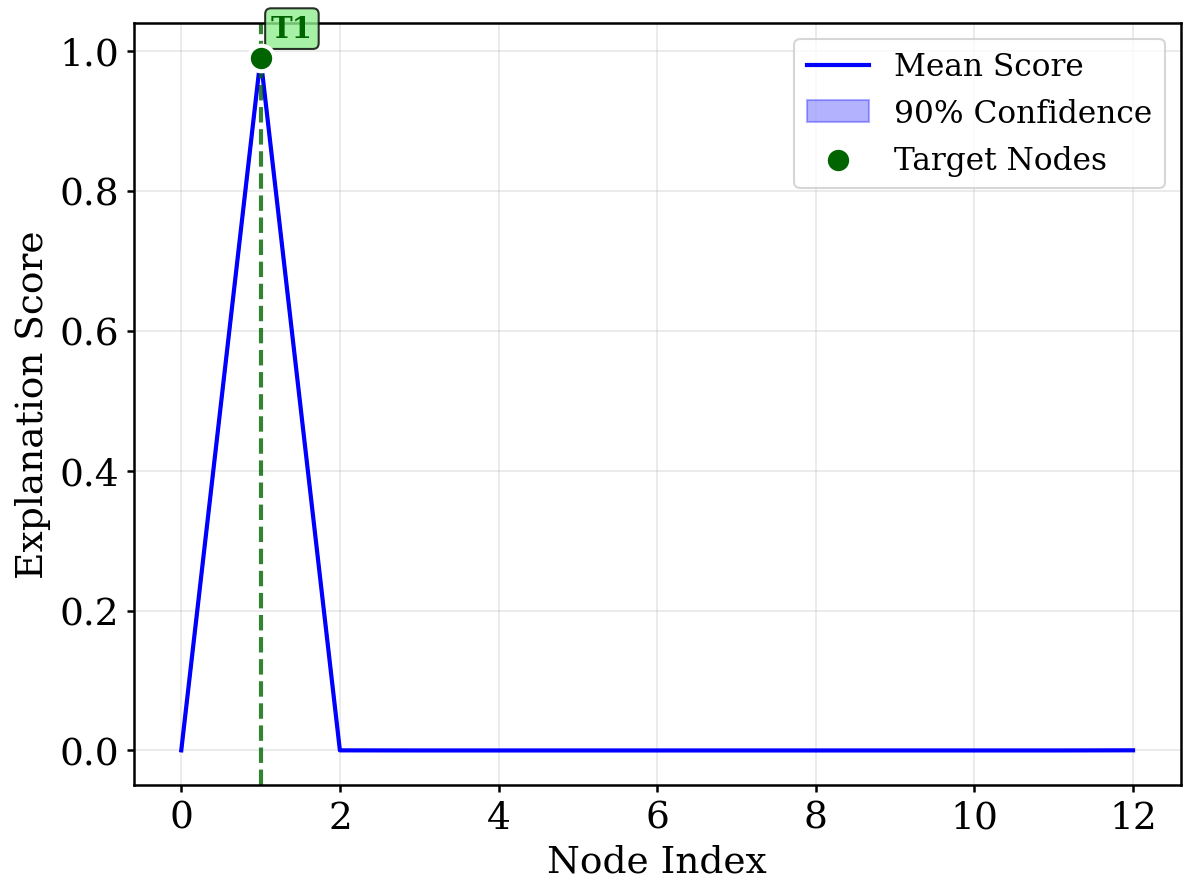}
    \end{subfigure}

    \vspace{0.6cm}

    \begin{subfigure}{0.49\textwidth}
        \includegraphics[width=\linewidth]{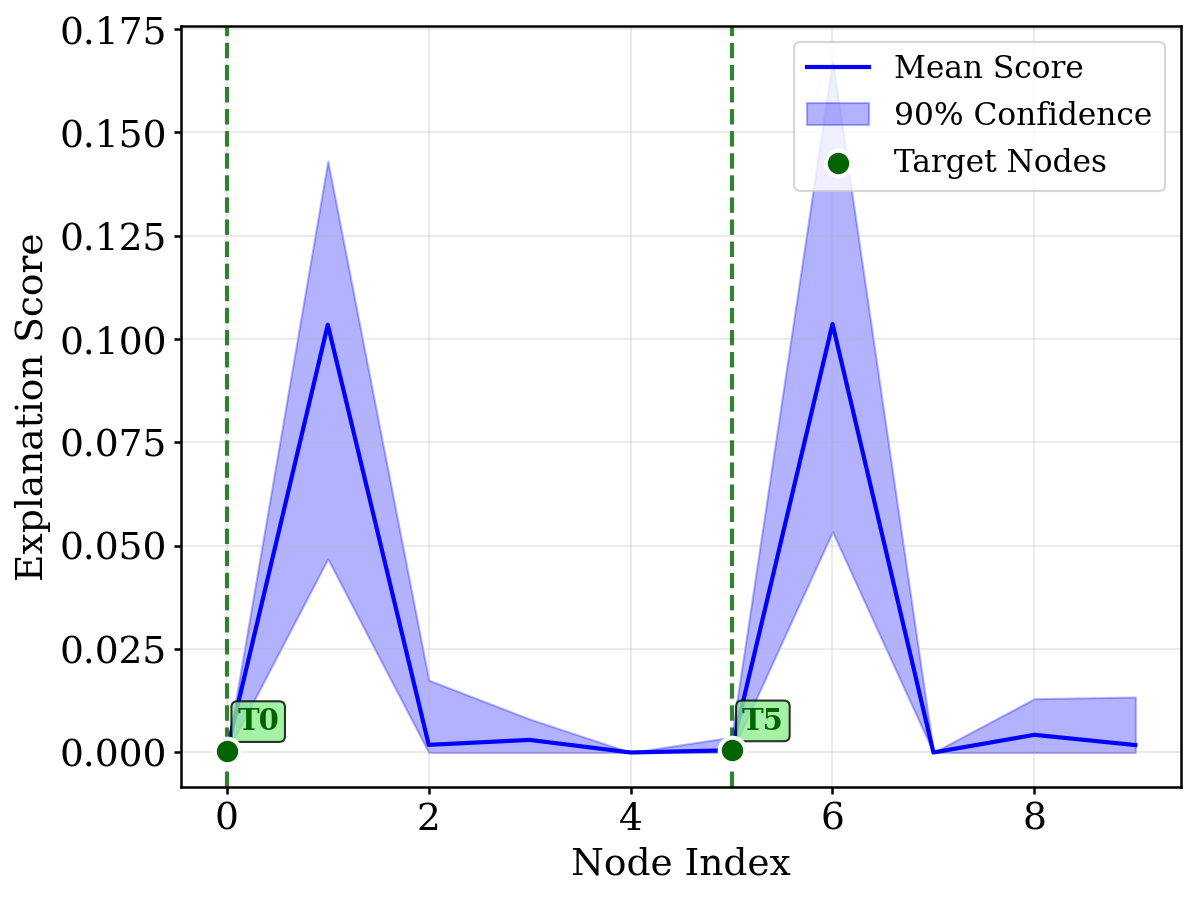}
    \end{subfigure}%
    \hfill
    \begin{subfigure}{0.49\textwidth}
        \includegraphics[width=\linewidth]{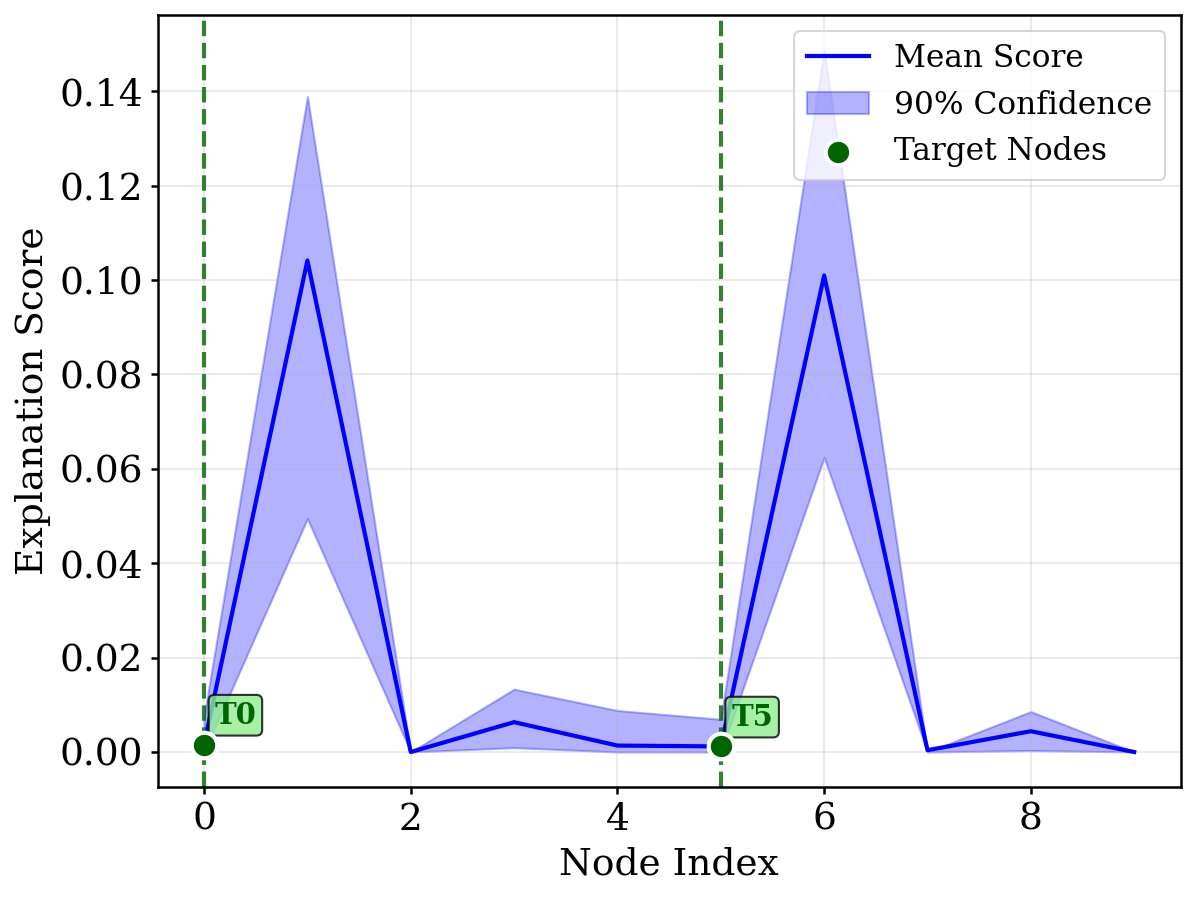}
    \end{subfigure}

   \caption{Effect of Different Perturbation Strategies: Explanation Score variation across surrogates per node due to quantum stochasticity for Case 1 (top) and Case 2 (bottom) under Random Walk perturbation. \bfa{Left} - HSIC-L1 surrogate; \bfa{Right} - HSIC-Group surrogate.}
    
    \label{fig:perturbation_rw}
\end{figure}

\begin{figure}[ht]
    \centering

    \begin{subfigure}{0.49\textwidth}
        \includegraphics[width=\linewidth]{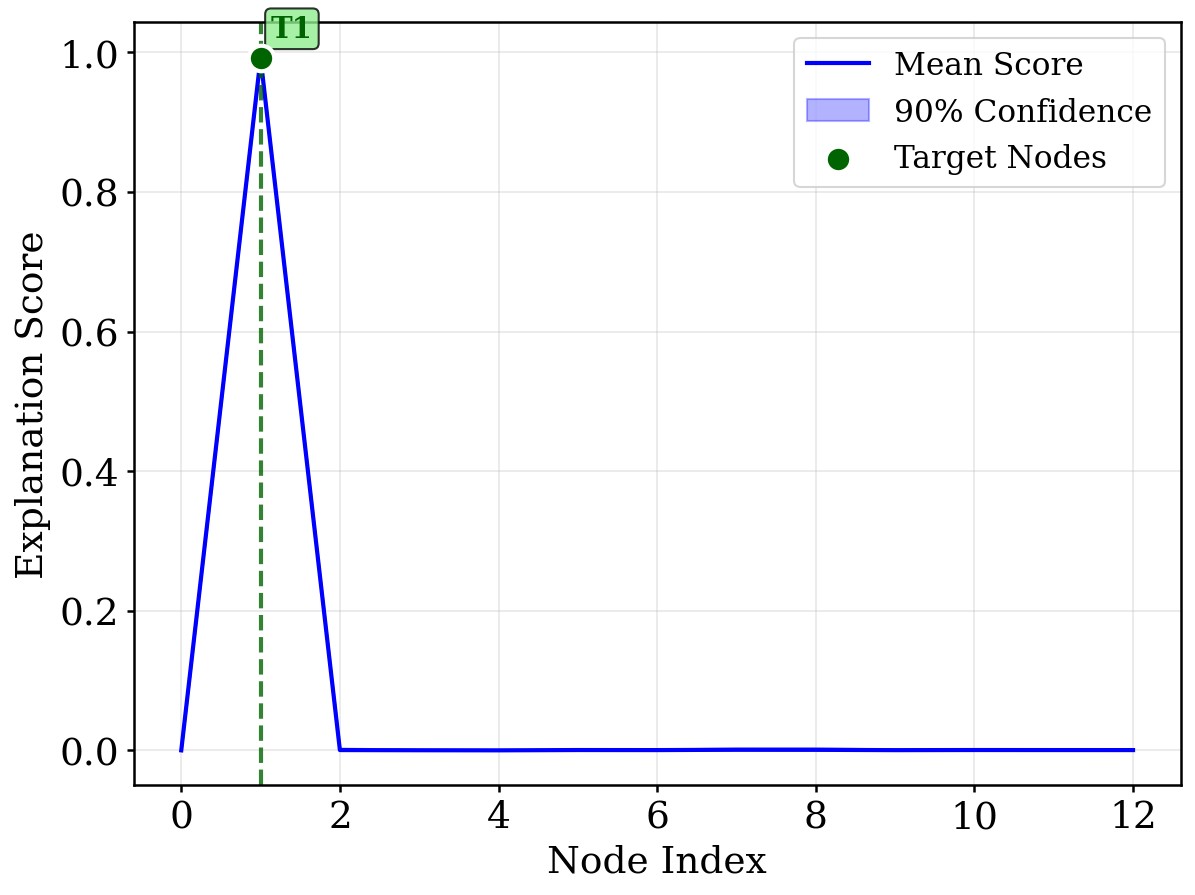}
    \end{subfigure}%
    \hfill
    \begin{subfigure}{0.49\textwidth}
        \includegraphics[width=\linewidth]{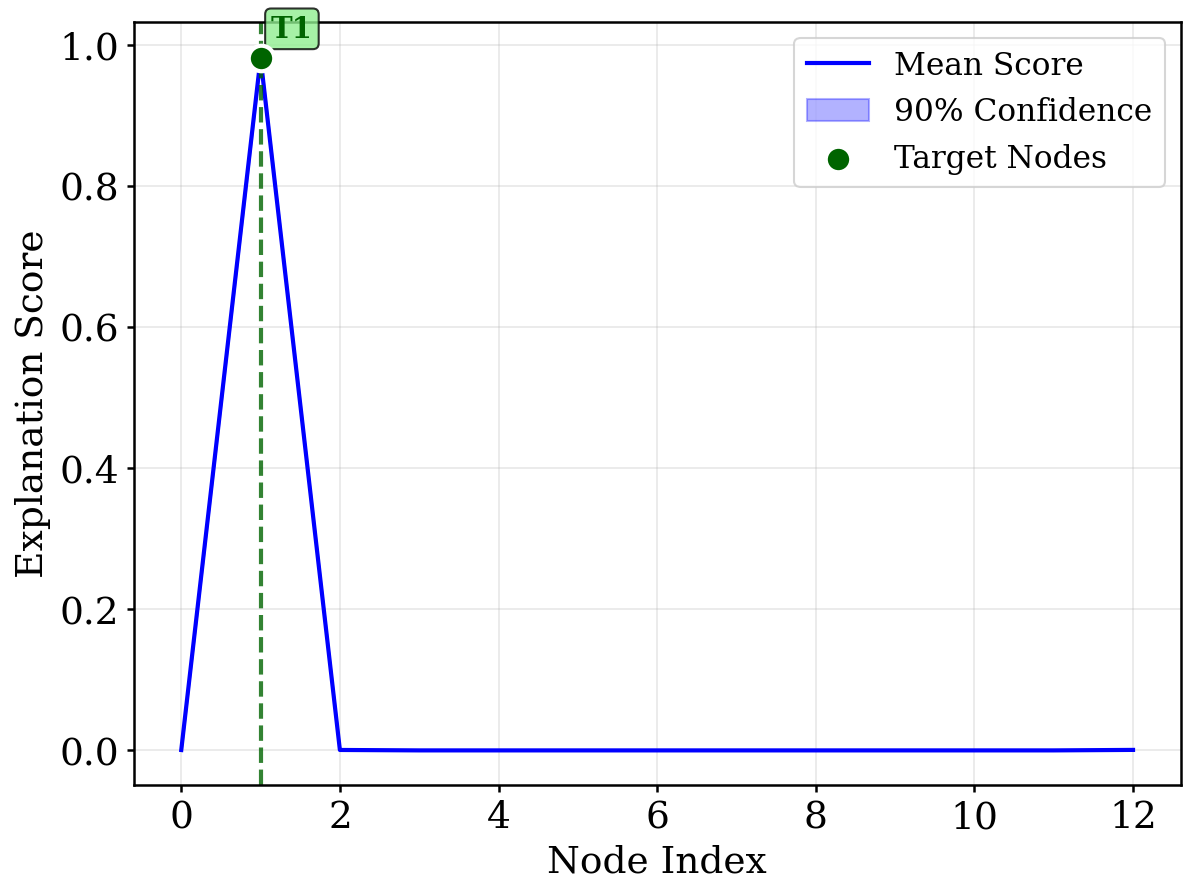}
    \end{subfigure}

    \vspace{0.6cm}

    \begin{subfigure}{0.49\textwidth}
        \includegraphics[width=\linewidth]{sections/images/results/case_2/case_1_2_eqgc_QLIME_HSIC-L1_uncertainty_analysis_graph_9_confidence_bounds.png}
    \end{subfigure}%
    \hfill
    \begin{subfigure}{0.49\textwidth}
        \includegraphics[width=\linewidth]{sections/images/results/case_2/case_1_2_eqgc_QLIME_HSIC-Group_uncertainty_analysis_graph_9_confidence_bounds.png}
    \end{subfigure}

    \caption{Effect of Different Perturbation Strategies:
        Explanation Score variation across surrogates per node due to quantum stochasticity for Case 1 (top) and Case 2 (bottom) under Random Node perturbation. \bfa{Left} - HSIC-L1 surrogate; \bfa{Right} - HSIC-Group surrogate.
    }
    \label{fig:perturbation_rn}
\end{figure}

\begin{table}[ht]
    \centering
    \caption{Effect of Different Perturbation Strategies: Explanation Accuracy - Random Walk Perturbation}
    \begin{tabular}{lcc|cccc}
        \toprule
        Method & \emph{One@1} & \emph{One@3} & \emph{Both@2} & \emph{Both@6} & \emph{One@2} & \emph{One@6} \\
        \midrule
        QGLIME (HSIC-L1) & $\best{1.00 \pm 0.00}$ & $\best{1.00 \pm 0.00}$ & $0.0 \pm 0.00$ & $\sbest{0.1 \pm 0.30}$ & $0.0 \pm 0.00$ & $\sbest{0.9 \pm 0.30}$ \\
        QGLIME (HSIC-G)  & $\best{1.00 \pm 0.00}$ & $\best{1.00 \pm 0.00}$ & $0.0 \pm 0.00$ & $\best{0.5 \pm 0.50}$ & $0.0 \pm 0.00$ & $\best{1.0 \pm 0.00}$ \\
        \bottomrule
    \end{tabular}
    \label{tab:combined_accuracy_rw}
\end{table}

\begin{table}[ht]
    \centering
    \caption{Effect of Different Perturbation Strategies: Metrics - Random Walk Perturbation}
    \begin{tabular}{lccccc}
        \toprule
        Method  & $\fidm$ & $\fidp$ & $\spars$ & $\cons$ & $\ri$ \\
        \midrule
        \multicolumn{6}{l}{\textbf{Case 1: Single-Target}} \\
        \addlinespace[0.3em]
        QGLIME (HSIC-L1) & $\sbest{0.957 \pm 0.028}$ & $\sbest{0.785 \pm 0.294}$ & $\best{0.875 \pm 0.051}$ & $\best{1.000}$ & $\sbest{16.48}$ \\
        QGLIME (HSIC-G)  & $\best{0.988 \pm 0.025}$  & $\best{0.793 \pm 0.300}$  & $\best{0.875 \pm 0.051}$  & $\best{1.000}$ & $\best{17.97}$ \\
        \addlinespace[0.2em]
        \midrule
        \multicolumn{6}{l}{\textbf{Case 2: Dual-Target}} \\
        \addlinespace[0.3em]
        QGLIME (HSIC-L1) & $\best{0.114 \pm 0.074}$ & $\best{0.463 \pm 0.461}$ & $\best{0.762 \pm 0.048}$ & $\sbest{0.04}$ & $\sbest{0.017}$ \\
        QGLIME (HSIC-G)  & $\best{0.114 \pm 0.074}$  & $\best{0.463 \pm 0.461}$  & $\sbest{0.756 \pm 0.053}$ & $\best{0.06}$ & $\best{0.051}$ \\
        \bottomrule
    \end{tabular}
    \label{tab:combined_metrics_rw}
\end{table}

Figures~\ref{fig:perturbation_rw}--\ref{fig:perturbation_rn} and Tables~\ref{tab:combined_accuracy_rw}--\ref{tab:combined_metrics_rw} compare random-walk and random-node perturbations across single-target (Case 1) and dual-target (Case 2) settings. In single-target settings, random-walk perturbations produce highly focused and reproducible explanations, yielding strong scores across surrogate models while maintaining sparsity. This demonstrates that leveraging local graph connectivity stabilizes node importance estimates and isolates the key target effectively. In contrast, multi-target scenarios reveal that random-walk perturbations, while still producing high-explanation scores for selected nodes, tend to underrepresent some influential targets, reducing explanation completeness. Random node perturbations distribute importance more evenly, better capturing all relevant nodes, though with slightly lower per-node confidence. Collectively, these results highlight a tradeoff in perturbation strategy: random-walk perturbations favor sparse, high-confidence explanations suitable for single-target identification, but can compromise coverage and robustness in multi-target contexts. Aligning perturbation choice with the complexity of the explanation task is therefore critical for achieving reliable and interpretable QGraphLIME explanations.

\subsubsection{Effect of Surrogate (Non-)Linearity}
\begin{figure}[ht]
    \centering
    \begin{subfigure}{0.30\textwidth}
        \includegraphics[width=\linewidth]{sections/images/results/case_1/graph_13.png}
    \end{subfigure}%
    \begin{subfigure}{0.33\textwidth}
        \includegraphics[width=\linewidth]{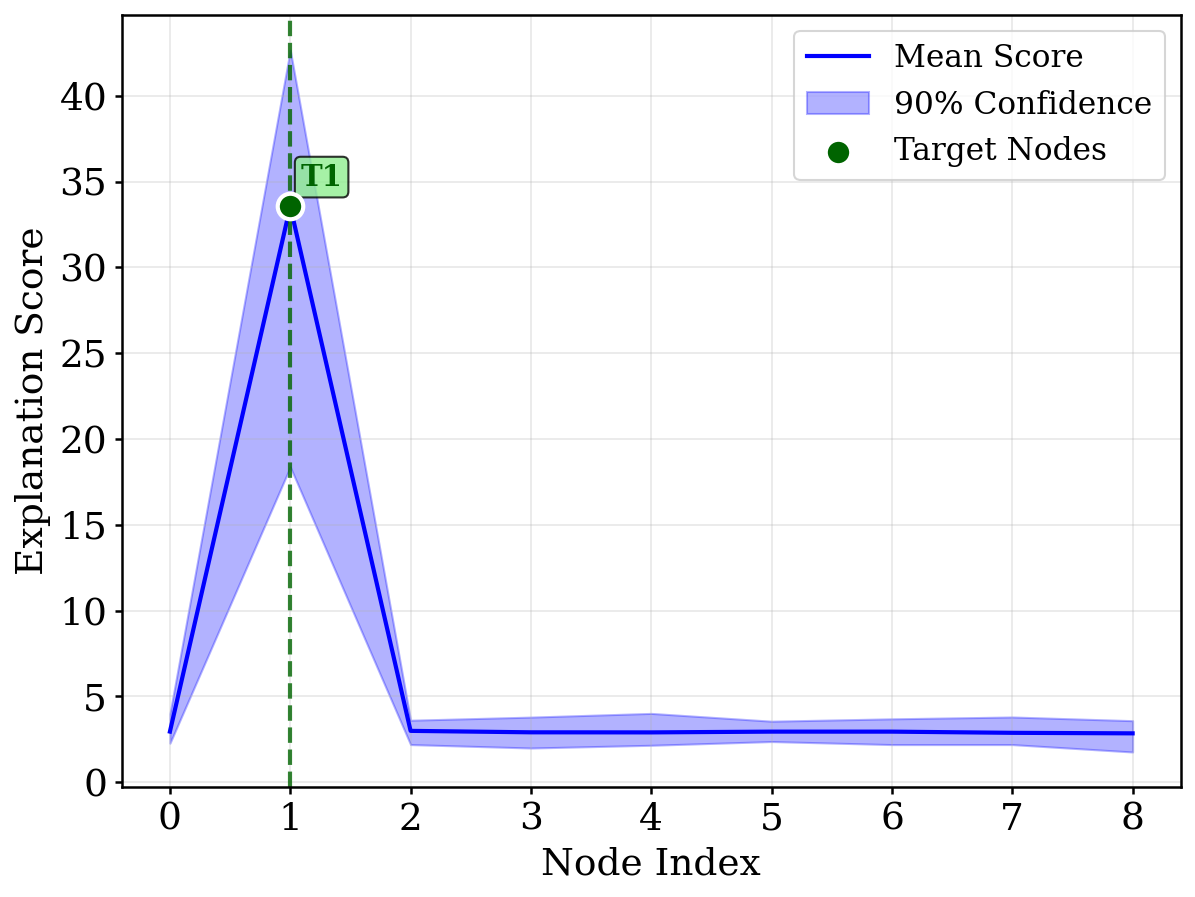}
    \end{subfigure}%
    \begin{subfigure}{0.33\textwidth}
        \includegraphics[width=\linewidth]{sections/images/results/case_1/case_1_1_eqgc_QLIME_HSIC-Group_uncertainty_analysis_graph_13_confidence_bounds.png}
    \end{subfigure}

    \vspace{0.4cm}

    \begin{subfigure}{0.30\textwidth}
        \includegraphics[width=\linewidth]{sections/images/results/case_1/graph_25.png}
    \end{subfigure}%
    \begin{subfigure}{0.33\textwidth}
        \includegraphics[width=\linewidth]{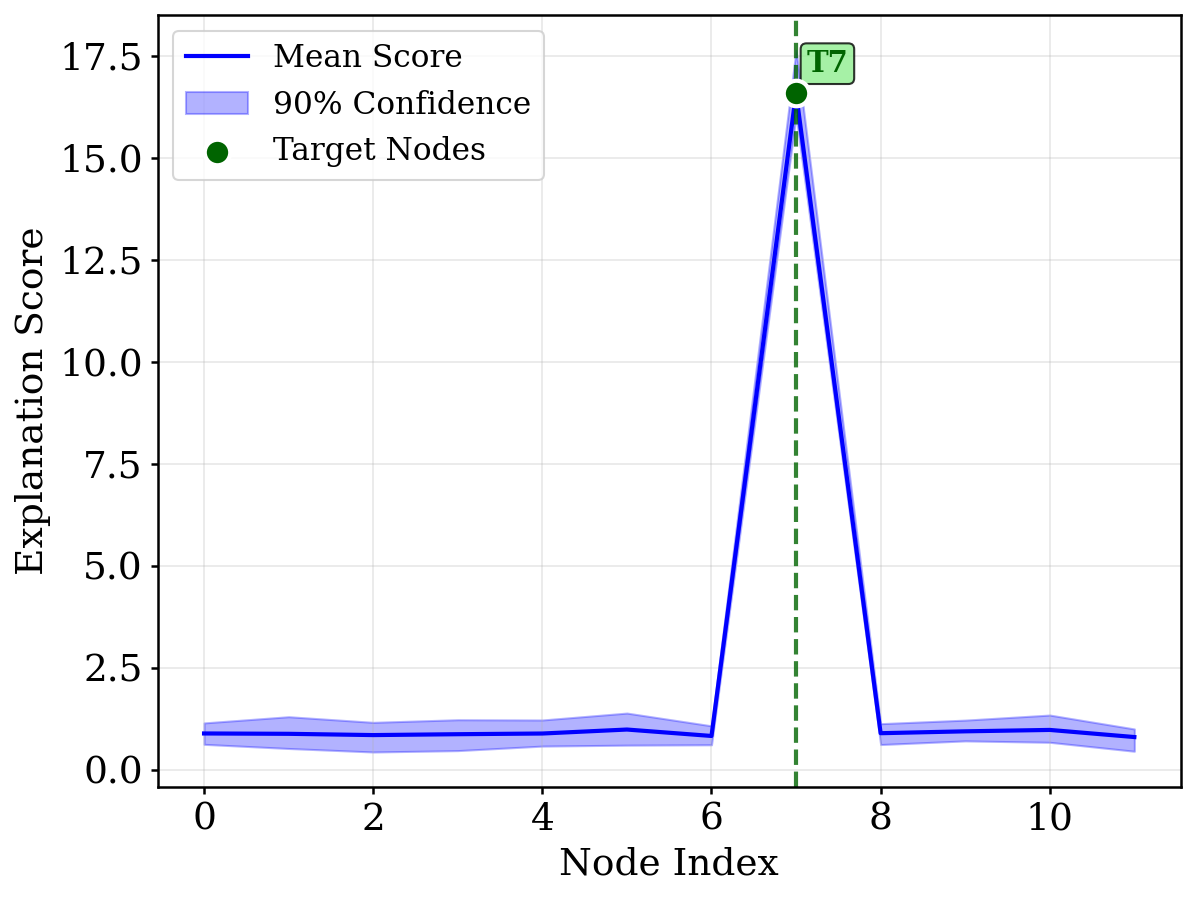}
    \end{subfigure}%
    \begin{subfigure}{0.33\textwidth}
        \includegraphics[width=\linewidth]{sections/images/results/case_1/case_1_1_eqgc_QLIME_HSIC-Group_uncertainty_analysis_graph_25_confidence_bounds.png}
    \end{subfigure}

   \caption{
     Effect of Surrogate (Non-)Linearity: Explanation Score variation across surrogates per node due to quantum stochasticity for Case 1 - Single-Target. \bfa{Left} - Input Graph; \bfa{Center} - Logistic surrogate ; \bfa{Right} - HSIC-Group surrogate.
    }
    \label{fig:lin1}
\end{figure}

\begin{figure}[ht]
    \centering
    \begin{subfigure}{0.30\textwidth}
        \includegraphics[width=\linewidth]{sections/images/results/case_2/graph_27.png}
    \end{subfigure}%
    \begin{subfigure}{0.33\textwidth}
        \includegraphics[width=\linewidth]{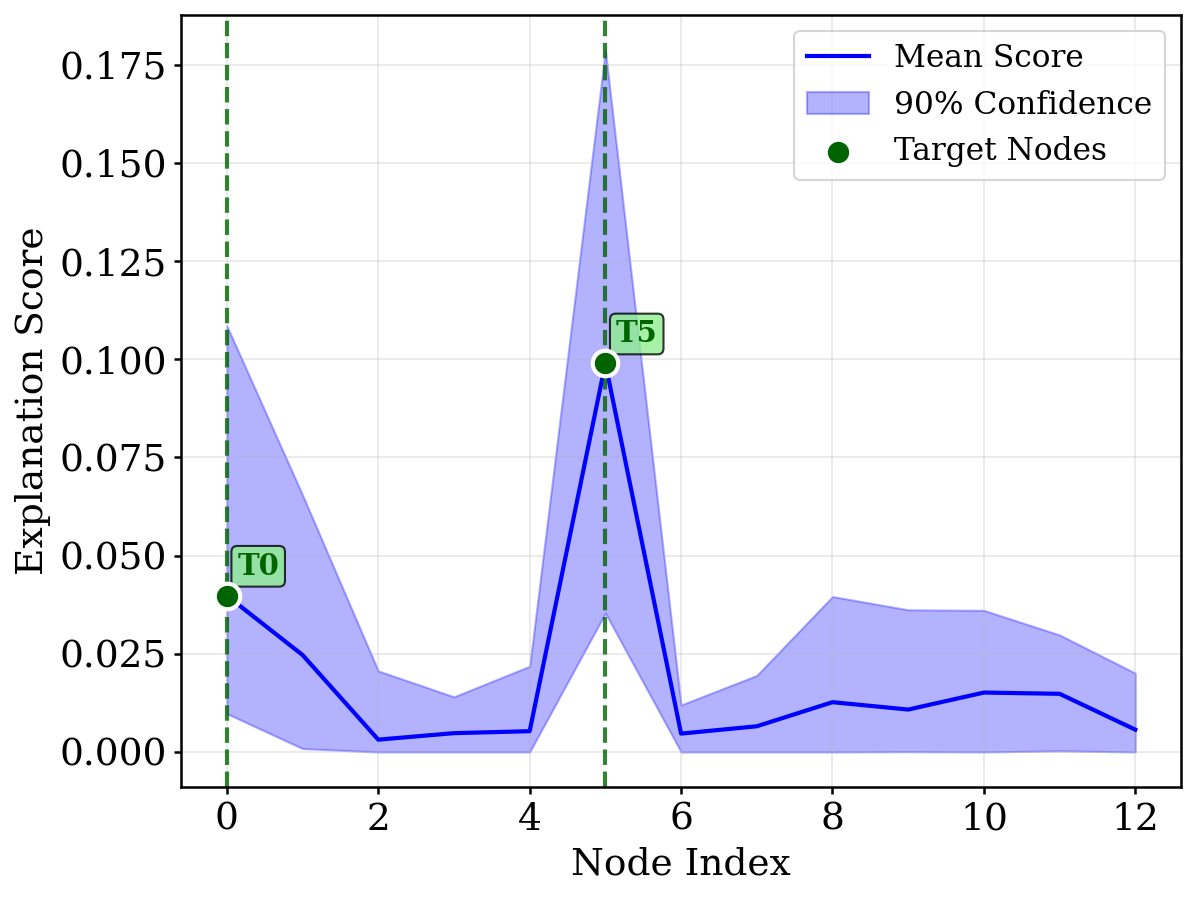}
    \end{subfigure}%
    \begin{subfigure}{0.33\textwidth}
        \includegraphics[width=\linewidth]{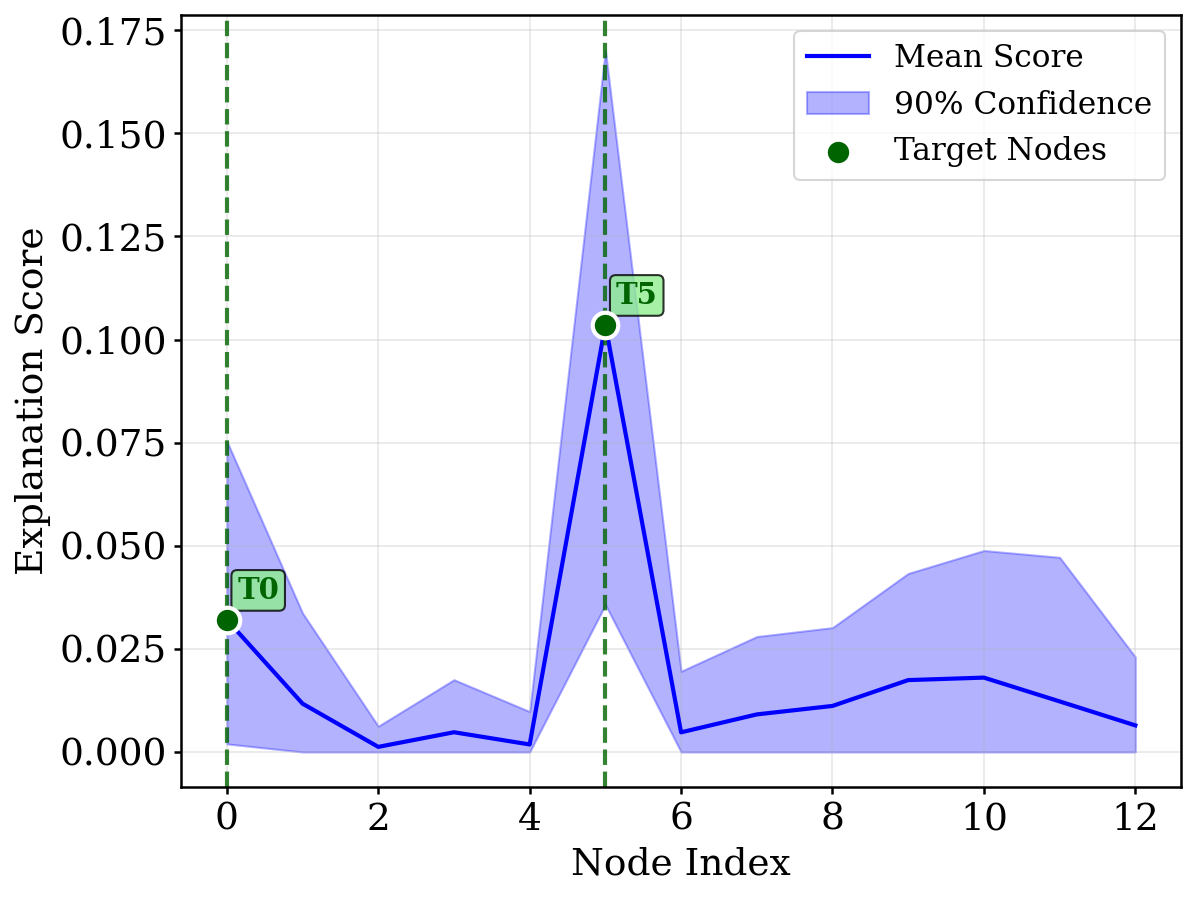}
    \end{subfigure}

    \vspace{0.4cm}

    \begin{subfigure}{0.30\textwidth}
        \includegraphics[width=\linewidth]{sections/images/results/case_2/graph_9.png}
    \end{subfigure}%
    \begin{subfigure}{0.33\textwidth}
        \includegraphics[width=\linewidth]{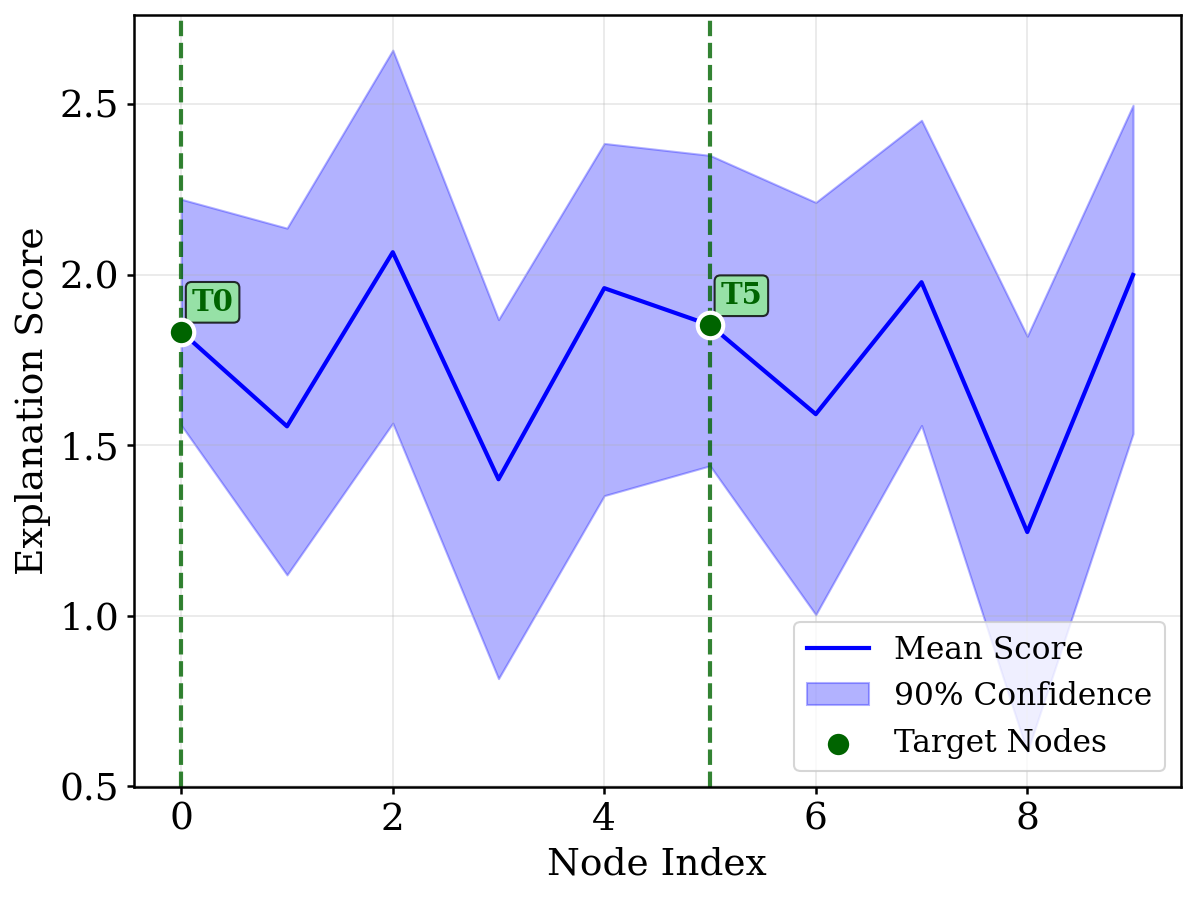}
    \end{subfigure}%
    \begin{subfigure}{0.33\textwidth}
        \includegraphics[width=\linewidth]{sections/images/results/case_2/case_1_2_eqgc_QLIME_HSIC-Group_uncertainty_analysis_graph_9_confidence_bounds.png}
    \end{subfigure}

    \caption{
     Effect of Surrogate (Non-)Linearity: Explanation Score variation across surrogates per node due to quantum stochasticity for Case 2 - Single-Target. \bfa{Left} - Input Graph; \bfa{Center} - Logistic surrogate ; \bfa{Right} - HSIC-Group surrogate.
    }
    \label{fig:lin2}
\end{figure}

\begin{table}[ht]
    \centering
    \caption{Effect of Surrogate (Non-)Linearity: Explanation Accuracy (Case 1 and Case 2)}
    \begin{tabular}{lcc|cccc}
        \toprule
        Method & \emph{One@1} & \emph{One@3} & \emph{Both@2} & \emph{Both@6} & \emph{One@2} & \emph{One@6} \\
        \midrule
        QGLIME (Logistic) & $\sbest{0.90 \pm 0.30}$ & $\best{1.00 \pm 0.00}$ & $\best{0.90 \pm 0.30}$ & $\best{1.00 \pm 0.00}$ & $\best{0.90 \pm 0.30}$ & $\best{1.00 \pm 0.00}$ \\
        QGLIME (HSIC-L1) & $\best{1.00 \pm 0.00}$  & $\best{1.00 \pm 0.00}$  & $\best{0.90 \pm 0.30}$ & $\best{1.00 \pm 0.00}$ & $\best{0.90 \pm 0.30}$ & $\best{1.00 \pm 0.00}$ \\
        QGLIME (HSIC-G)  & $\best{1.00 \pm 0.00}$  & $\best{1.00 \pm 0.00}$  & $\best{0.90 \pm 0.30}$ & $\sbest{0.90 \pm 0.30}$ & $\best{0.90 \pm 0.30}$ & $\sbest{0.90 \pm 0.30}$ \\
        \bottomrule
    \end{tabular}
    \label{tab:combined_accuracy_lin}
\end{table}

\begin{table}[ht]
\centering
\caption{Effect of Surrogate (Non-)Linearity: Metrics (Case 1 and Case 2) }
\begin{tabular}{lccc}
    \toprule
    Method  & $\fidm$ & $\fidp$& $\spars$ \\
    \midrule
    \multicolumn{4}{l}{\textbf{Case 1: Single-Target}} \\
    \addlinespace[0.3em]
    QGLIME (Logistic) & $\sbest{0.969 \pm 0.031}$ & $\best{0.924 \pm 0.197}$ & $\sbest{0.726 \pm 0.363}$ \\
    QGLIME (HSIC-L1) & $0.963 \pm 0.030$ & $0.844 \pm 0.256$ & $\sbest{0.726 \pm 0.363}$ \\
    QGLIME (HSIC-G)  & $\best{0.975 \pm 0.030}$  & $\sbest{0.866 \pm 0.265}$  & $\best{0.746 \pm 0.327}$ \\
    \addlinespace[0.3em]
    \midrule
    \multicolumn{4}{l}{\textbf{Case 2: Dual-Target}} \\
    \addlinespace[0.3em]
    QGLIME (Logistic) & $\best{0.095 \pm 0.282}$  & $\best{0.900 \pm 0.300}$  & $0.000 \pm 0.000$ \\
    QGLIME (HSIC-L1) & $\sbest{0.001 \pm 0.0000}$ & $0.863 \pm 0.290$ & $\best{0.544 \pm 0.231}$ \\
    QGLIME (HSIC-G)  & $\sbest{0.001 \pm 0.0000}$ & $\sbest{0.873 \pm 0.285}$  & $\sbest{0.467 \pm 0.252}$ \\
    \bottomrule
\end{tabular}
\label{tab:combined_metrics_lin}
\end{table}

Figures~\ref{fig:lin1}, \ref{fig:lin2} and Tables~\ref{tab:combined_accuracy_lin}, \ref{tab:combined_metrics_lin} illustrate the impact of surrogate model choice. In single-target scenarios, these nonlinear models achieve perfect target identification with minimal variability, accurately reflecting the local decision boundaries of the QGNN. While linear surrogates can occasionally attain comparable fidelity, their lower confidence metrics indicate reduced reproducibility across surrogate instances. Sparsity is comparable across all surrogates, indicating that the improved reliability of nonlinear models does not compromise explanation conciseness. The benefits of nonlinear surrogates are even more pronounced in dual-target scenarios. They sustain high accuracy, retain meaningful $\fidm$ values, and maintain moderate sparsity, producing explanations that prioritize the most relevant nodes while remaining statistically robust. Linear surrogates, by contrast, exhibit lower $\fidp$ and near-zero sparsity, suggesting less focused and more dispersed attributions. Confidence metrics further confirm that nonlinear models yield explanations that are reproducible and stable, even in complex multi-target contexts. Collectively, these observations highlight that HSIC-regularized nonlinear surrogates are better suited for QGNN interpretability, offering an optimal balance of accuracy, stability, and interpretability, particularly when multiple influential nodes contribute to the model's predictions.

\subsubsection{Effect of Measurement: Single-shot StGraphLIME}
\begin{figure}[ht]
    \centering

    \begin{subfigure}{0.32\textwidth}
        \includegraphics[width=\linewidth]{sections/images/results/case_1/graph_13.png}
    \end{subfigure}%
    \hfill
    \begin{subfigure}{0.32\textwidth}
        \includegraphics[width=\linewidth]{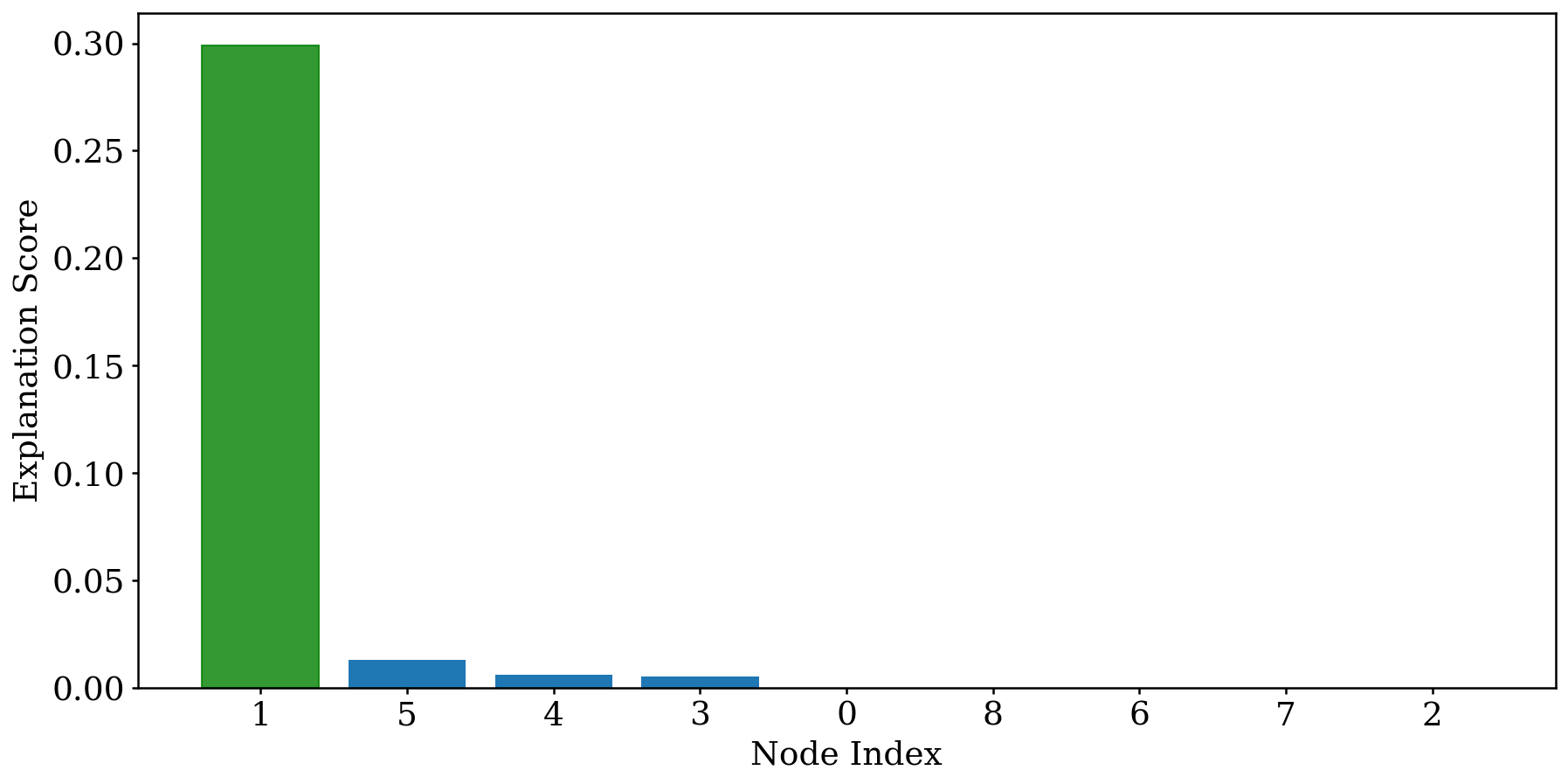}
    \end{subfigure}%
    \hfill
    \begin{subfigure}{0.32\textwidth}
        \includegraphics[width=\linewidth]{sections/images/results/case_1/case_1_1_eqgc_QLIME_HSIC-L1_uncertainty_analysis_graph_13_confidence_bounds.png}
    \end{subfigure}

    \vspace{0.6cm}

    \begin{subfigure}{0.32\textwidth}
        \includegraphics[width=\linewidth]{sections/images/results/case_1/graph_25.png}
    \end{subfigure}%
    \hfill
    \begin{subfigure}{0.32\textwidth}
        \includegraphics[width=\linewidth]{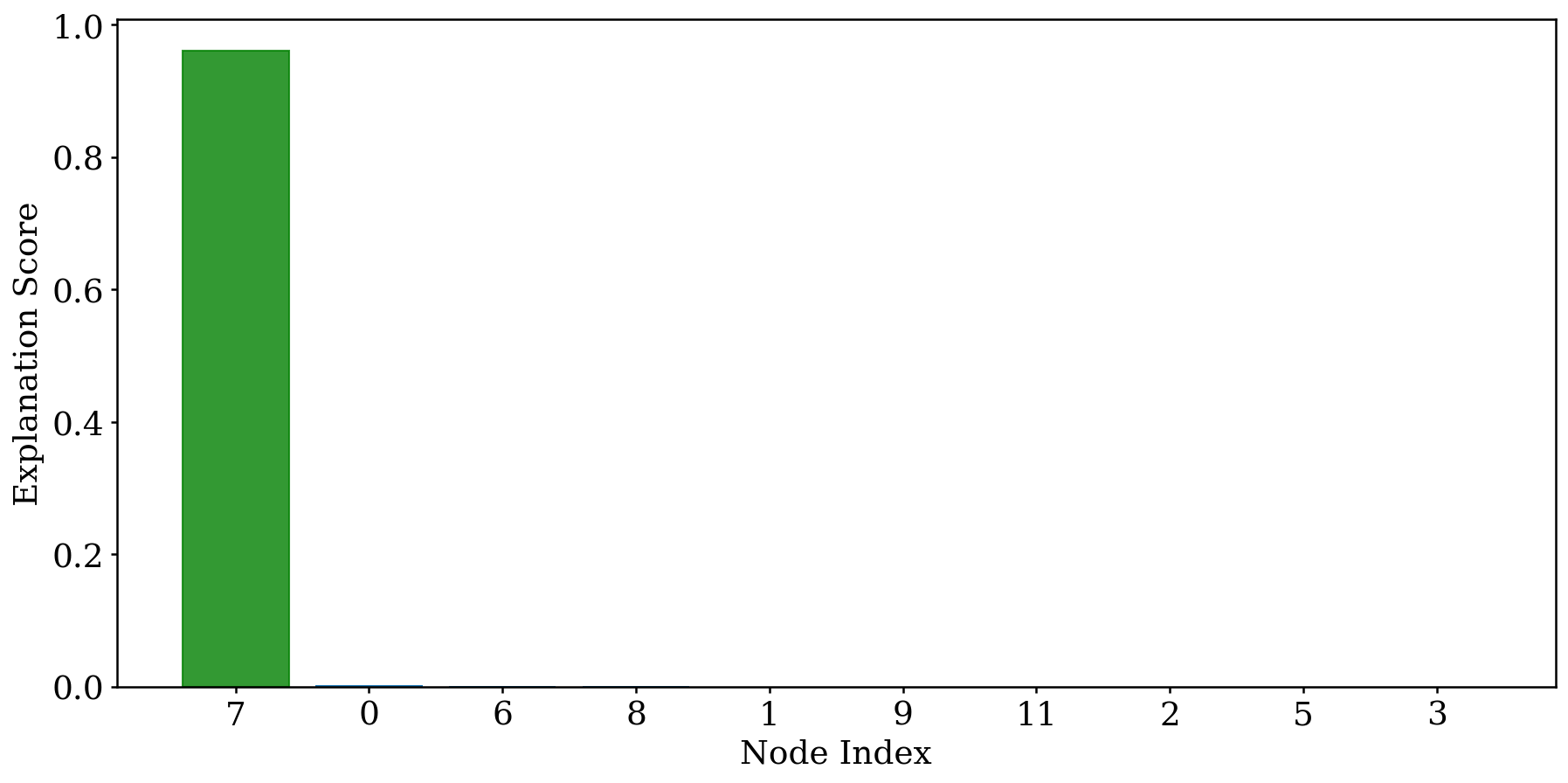}
    \end{subfigure}%
    \hfill
    \begin{subfigure}{0.32\textwidth}
        \includegraphics[width=\linewidth]{sections/images/results/case_1/case_1_1_eqgc_QLIME_HSIC-L1_uncertainty_analysis_graph_25_confidence_bounds.png}
    \end{subfigure}

   \caption{Effect of Measurement: StGraphLIME on QGNN: Explanation Score variation across surrogates per node due to quantum stochasticity for Case 1 - Single-Target. \bfa{Left} - Input Graph;\bfa{Center} - StGraphLIME-HSIC-L1 ; \bfa{Right} - QGLIME-HSIC-L1
    }
    \label{fig:case1_graph_stq_hsicL1}
\end{figure}

\begin{figure}[ht]
    \centering

    \begin{subfigure}{0.32\textwidth}
        \includegraphics[width=\linewidth]{sections/images/results/case_1/graph_13.png}
    \end{subfigure}%
    \hfill
    \begin{subfigure}{0.32\textwidth}
        \includegraphics[width=\linewidth]{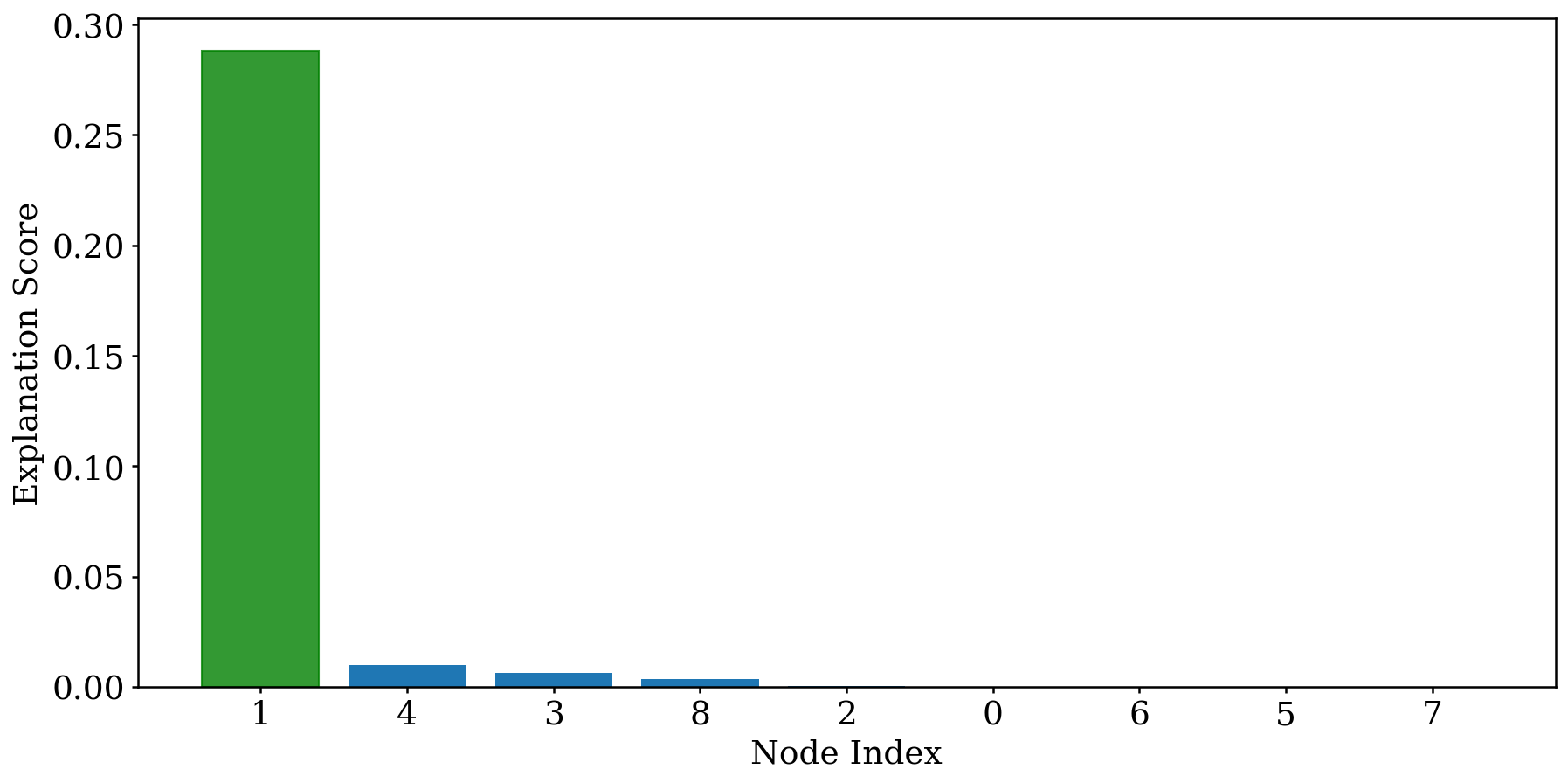}
    \end{subfigure}%
    \hfill
    \begin{subfigure}{0.32\textwidth}
        \includegraphics[width=\linewidth]{sections/images/results/case_1/case_1_1_eqgc_QLIME_HSIC-Group_uncertainty_analysis_graph_31_confidence_bounds.png}
    \end{subfigure}

    \vspace{0.6cm}

    \begin{subfigure}{0.32\textwidth}
        \includegraphics[width=\linewidth]{sections/images/results/case_1/graph_25.png}
    \end{subfigure}%
    \hfill
    \begin{subfigure}{0.32\textwidth}
        \includegraphics[width=\linewidth]{sections/images/results/case_2/7st.png}
    \end{subfigure}%
    \hfill
    \begin{subfigure}{0.32\textwidth}
        \includegraphics[width=\linewidth]{sections/images/results/case_1/case_1_1_eqgc_QLIME_HSIC-Group_uncertainty_analysis_graph_25_confidence_bounds.png}
    \end{subfigure}

      \caption{Effect of Measurement: StGraphLIME on QGNN: Explanation Score variation across surrogates per node due to quantum stochasticity for Case 1 - Single-Target. \bfa{Left} - Input Graph;\bfa{Center} - StGraphLIME-HSIC-G; \bfa{Right} - QGLIME-HSIC-G
    }
    \label{fig:case1_graph_stq_hsicGroup}
\end{figure}

\begin{figure}[ht]
    \centering

    \begin{subfigure}{0.32\textwidth}
        \includegraphics[width=\linewidth]{sections/images/results/case_2/graph_27.png}
    \end{subfigure}%
    \hfill
    \begin{subfigure}{0.32\textwidth}
        \includegraphics[width=\linewidth]{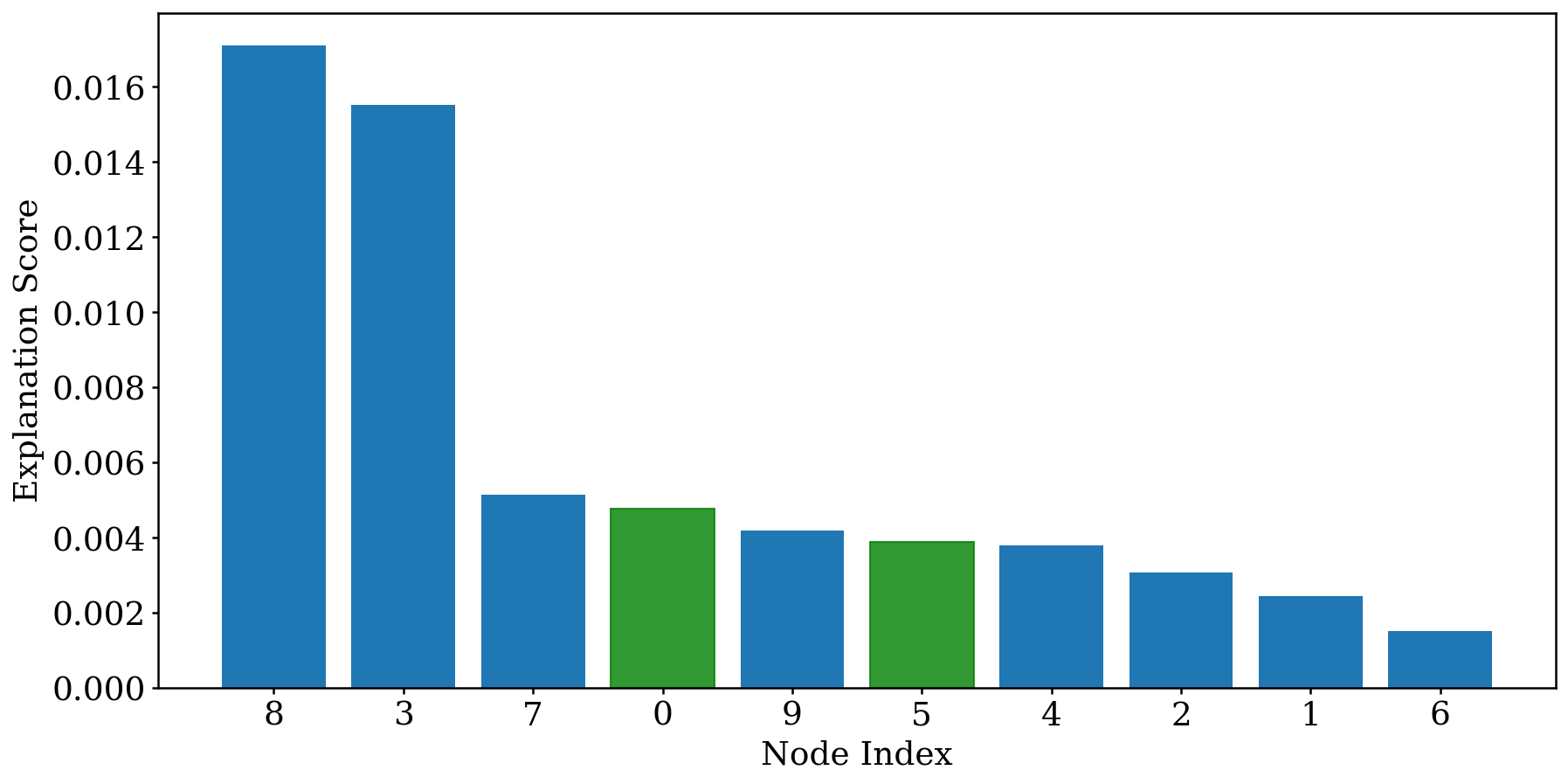}
    \end{subfigure}%
    \hfill
    \begin{subfigure}{0.32\textwidth}
        \includegraphics[width=\linewidth]{sections/images/results/case_2/q05l.png}
    \end{subfigure}

    \vspace{0.6cm}

    \begin{subfigure}{0.32\textwidth}
        \includegraphics[width=\linewidth]{sections/images/results/case_2/graph_9.png}
    \end{subfigure}%
    \hfill
    \begin{subfigure}{0.32\textwidth}
        \includegraphics[width=\linewidth]{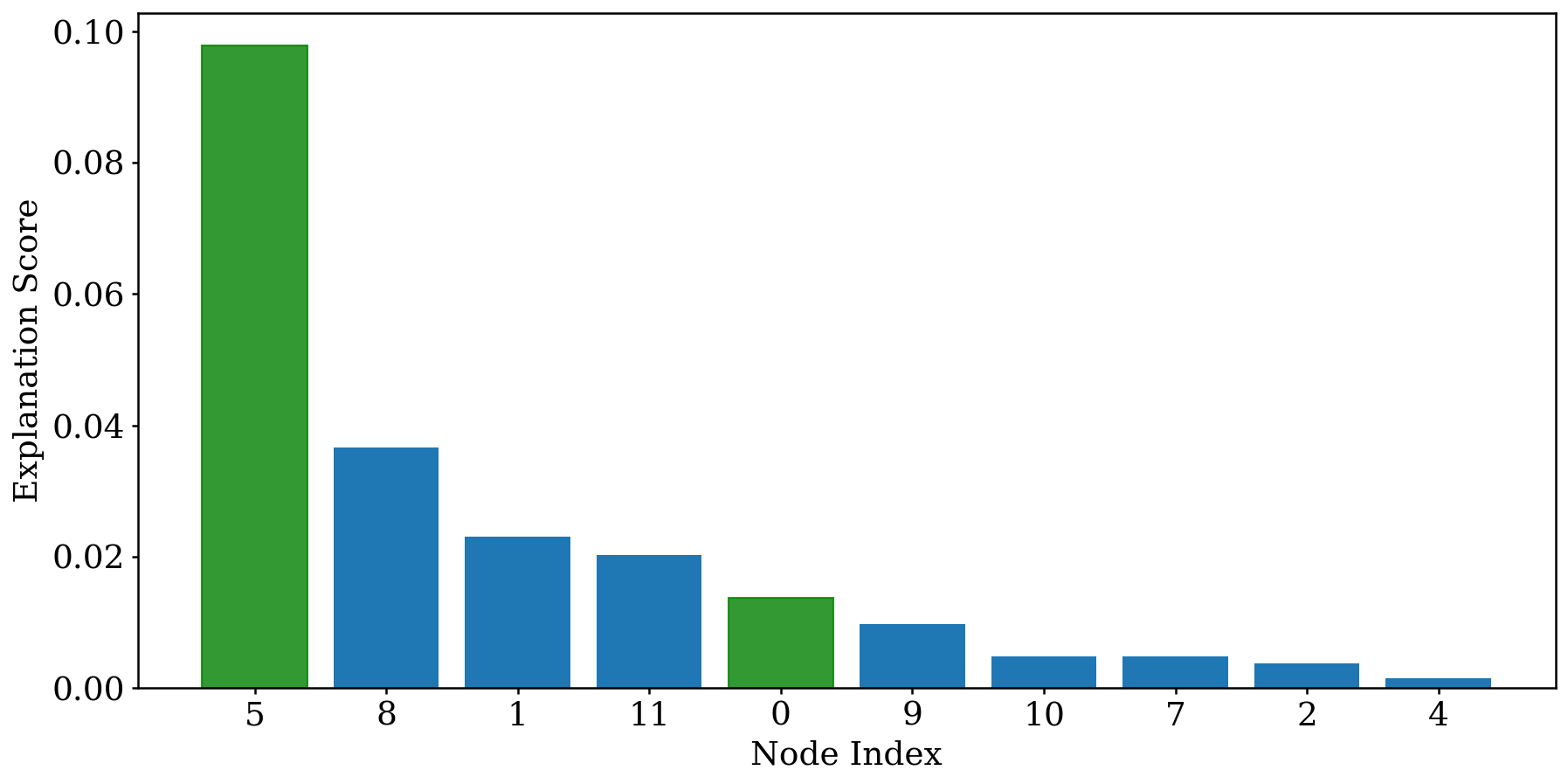}
    \end{subfigure}%
    \hfill
    \begin{subfigure}{0.32\textwidth}
        \includegraphics[width=\linewidth]{sections/images/results/case_2/case_1_2_eqgc_QLIME_HSIC-L1_uncertainty_analysis_graph_9_confidence_bounds.png}
    \end{subfigure}

    \caption{Effect of Measurement: StGraphLIME on QGNN: Explanation Score variation across surrogates per node due to quantum stochasticity for Case 2 - Single-Target. \bfa{Left} - Input Graph;\bfa{Center} - StGraphLIME-HSIC-L1; \bfa{Right} - QGLIME-HSIC-L1
    }

    \label{fig:case2_graph_stq_hsicL1}
\end{figure}

\begin{figure}[ht]
    \centering

    \begin{subfigure}{0.32\textwidth}
        \includegraphics[width=\linewidth]{sections/images/results/case_2/graph_27.png}
    \end{subfigure}%
    \hfill
    \begin{subfigure}{0.32\textwidth}
        \includegraphics[width=\linewidth]{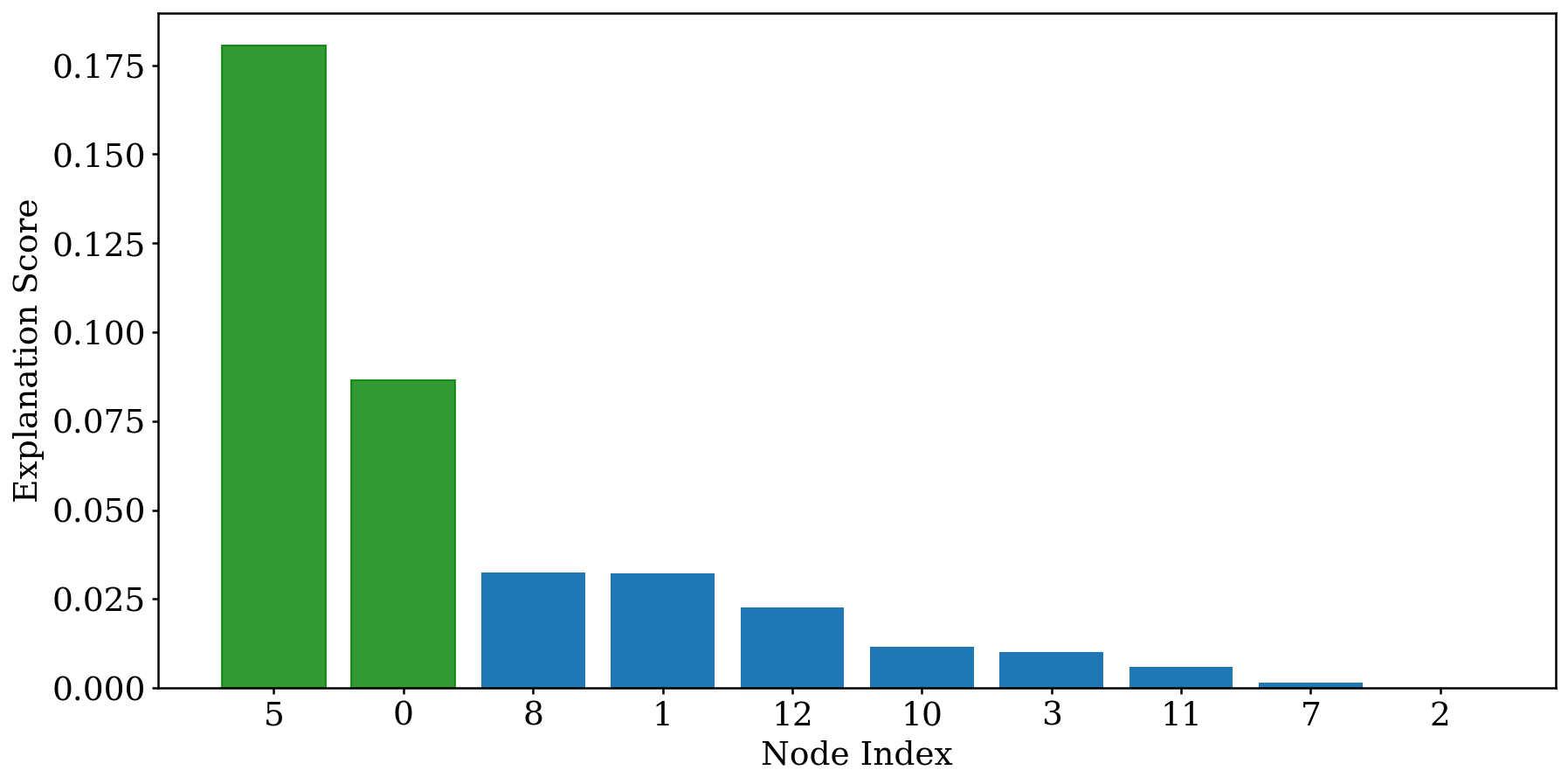}
    \end{subfigure}%
    \hfill
    \begin{subfigure}{0.32\textwidth}
        \includegraphics[width=\linewidth]{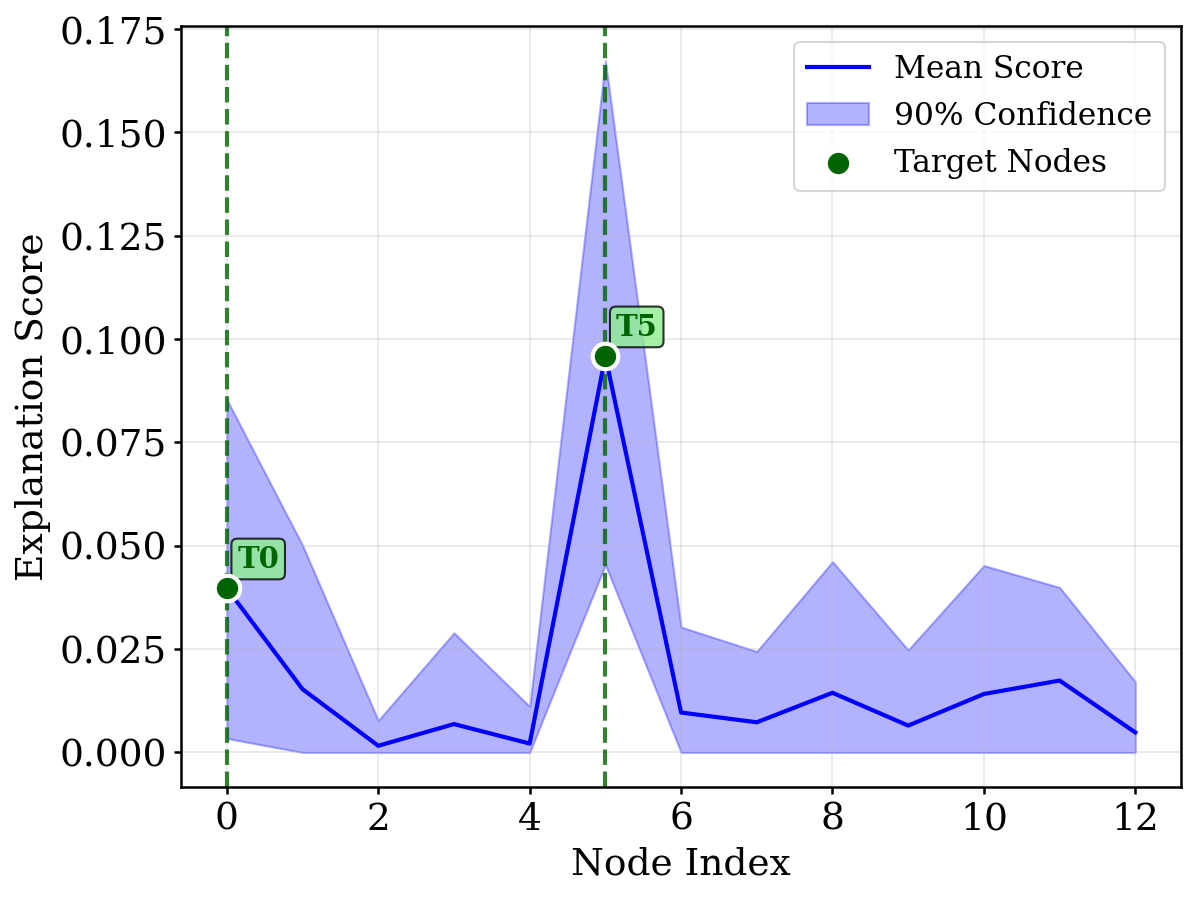}
    \end{subfigure}

    \vspace{0.6cm}

    \begin{subfigure}{0.32\textwidth}
        \includegraphics[width=\linewidth]{sections/images/results/case_2/graph_9.png}
    \end{subfigure}%
    \hfill
    \begin{subfigure}{0.32\textwidth}
        \includegraphics[width=\linewidth]{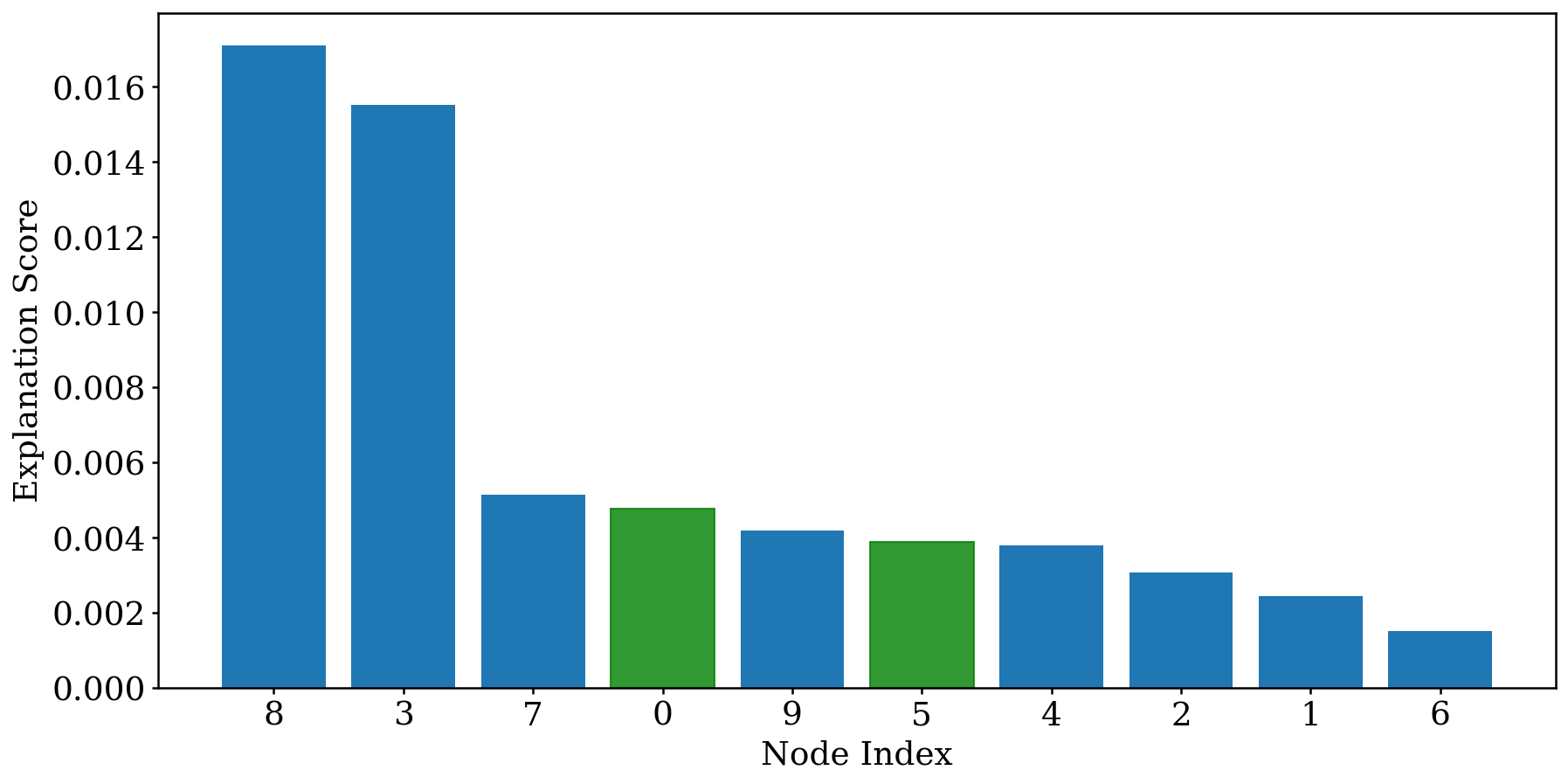}
    \end{subfigure}%
    \hfill
    \begin{subfigure}{0.32\textwidth}
        \includegraphics[width=\linewidth]{sections/images/results/case_2/case_1_2_eqgc_QLIME_HSIC-Group_uncertainty_analysis_graph_9_confidence_bounds.png}
    \end{subfigure}

      \caption{Effect of Measurement: StGraphLIME on QGNN: Explanation Score variation across surrogates per node due to quantum stochasticity for Case 1 - Single-Target. \bfa{Left} - Input Graph; \bfa{Center} - StGraphLIME-HSIC-G ; \bfa{Right} - QGLIME-HSIC-G
    }
    \label{fig:case2_graph_stq_hsicGroup}
\end{figure}

\begin{table}[h]
    \centering
    \caption{Effect of Measurement: StGraphLIME on QGNN - Explanation Accuracy: StGraphLIME vs QGLIME (HSIC-L1 and HSIC-G) under Random Node Perturbation (Case 1 and Case 2)}
    \begin{tabular}{lcc|cccc}
        \toprule
        Method & \emph{One@1} & \emph{One@3} & \emph{Both@2} & \emph{Both@6} & \emph{One@2} & \emph{One@6} \\
        \midrule
        StGLIME (HSIC-L1)    & $\best{1.00 \pm 0.00}$ & $\best{1.00 \pm 0.00}$ & $\sbest{0.40\pm 0.49}$ & $\best{1.00 \pm 0.00}$ & $\sbest{0.90 \pm 0.30}$ & $\best{1.00 \pm 0.00}$ \\
        StGLIME (HSIC-G) & $\best{1.00 \pm 0.00}$ & $\best{1.00 \pm 0.00}$ & $0.30 \pm 0.46$ & $0.80 \pm 0.40$ & $\best{1.00 \pm 0.00}$ & $\best{1.00 \pm 0.00}$ \\
        QGLIME (HSIC-L1)    & $\best{1.00 \pm 0.00}$ & $\best{1.00 \pm 0.00}$ & $\best{0.90 \pm 0.30}$ & $\best{1.00 \pm 0.00}$ & $\sbest{0.90 \pm 0.30}$ & $\best{1.00 \pm 0.00}$ \\
        QGLIME (HSIC-G)     & $\best{1.00 \pm 0.00}$ & $\best{1.00 \pm 0.00}$ & $\best{0.90 \pm 0.30}$ & $\sbest{0.90 \pm 0.30}$ & $\sbest{0.90 \pm 0.30}$ & $\sbest{0.90 \pm 0.30}$ \\
        \bottomrule
    \end{tabular}
    \label{tab:combined_accuracy_stq}
\end{table}

\begin{table}[ht]
\centering
\caption{Effect of Measurement: StGraphLIME on QGNN - Metrics: StGraphLIME vs QGLIME (HSIC-L1 and HSIC-G)}
\begin{tabular}{lccc}
    \toprule
    Method & $\fidm$ & $\fidp$ &  $\spars$ \\
    \midrule
    \multicolumn{4}{l}{\textbf{Case 1: Single-Target}} \\
    \addlinespace[0.3em]
    StGraphLIME (L1)    & $0.963 \pm 0.030$ & $0.866 \pm 0.264$ & $\sbest{0.786 \pm 0.248}$ \\
    StGraphLIME (Group) & $\sbest{0.969 \pm 0.031}$ & $\best{0.918 \pm 0.196}$ & $\best{0.795 \pm 0.221}$ \\
    QGLIME (HSIC-L1)    & $0.963 \pm 0.030$ & $0.844 \pm 0.256$ & $0.726 \pm 0.363$ \\
    QGLIME (HSIC-G)     & $\best{0.975 \pm 0.030}$ & $\sbest{0.866 \pm 0.265}$ & $0.746 \pm 0.327$ \\
    \addlinespace[0.3em]
    \midrule
    \multicolumn{4}{l}{\textbf{Case 2: Dual-Target}} \\
    \addlinespace[0.3em]
    StGraphLIME (L1)    & $\sbest{0.033 \pm 0.065}$ & $0.670 \pm 0.439$ & $\best{0.613 \pm 0.136}$ \\
    StGraphLIME (Group) & $\best{0.049 \pm 0.074}$ & $0.570 \pm 0.466$ & $\sbest{0.569 \pm 0.065}$ \\
    QGLIME (HSIC-L1)    & $0.001 \pm 0.000$ & $\sbest{0.863 \pm 0.290}$ & $0.544 \pm 0.231$ \\
    QGLIME (HSIC-G)     & $0.001 \pm 0.000$ & $\best{0.873 \pm 0.285}$ & $0.467 \pm 0.252$ \\
    \bottomrule
\end{tabular}
\label{tab:combined_metrics_stq}
\end{table}

Figures~\ref{fig:case1_graph_stq_hsicL1}--\ref{fig:case2_graph_stq_hsicGroup} and Tables~\ref{tab:combined_accuracy_stq} \& \ref{tab:combined_metrics_stq} compare single-shot StGraphLIME with multi-shot QGraphLIME, illustrating the impact of measurement strategy. In the single-target setting, both StGraphLIME and QGraphLIME achieve perfect target identification with minimal variability, demonstrating highly consistent and reliable explanations. While QGLIME with HSIC-G shows slightly better surrogate agreement when retaining key nodes, StGraphLIME excels in detecting the effect of node removal and produces more compact explanations, as reflected in higher sparsity values.

For dual-target scenarios, QGLIME's HSIC variants continue to outperform StGraphLIME, maintaining high accuracies and demonstrating robust $\fidm$ scores, whereas StGraphLIME exhibits lower and more variable accuracies. $\fidp$ remains modest across all methods, but StGraphLIME shows marginally higher values, highlighting some sensitivity to target retention. Sparsity levels remain comparable, although StGraphLIME tends to produce slightly more concise explanations.

Overall, these results indicate that while single-shot StGraphLIME can provide stable explanations in simpler, single-target cases, QGraphLIME's multi-shot approach is better suited for capturing complex multi-target dependencies. It should be noted that the reduced performance and variability observed in Case 2 partly stem from quantum hardware limitations, which constrain qubit counts and the expressiveness of the surrogate explanations for larger or more intricate graphs.

The limitation in Case 2 arises from quantum hardware constraints limiting qubit counts, which restrict the model's training and evaluation on larger, more complex datasets. This inherently constrains the expressiveness and generalizability of these explanation methods in multi-target cases. Consequently, the reduced performance variability and lower confidence metrics in Case 2 can be partly attributed to these limitations, highlighting the need for advanced quantum resources or simulation techniques to scale interpretability studies.

\subsubsection{Effect of Regularization on Surrogates}

\begin{table}[ht]
    \centering
    \caption{Effect of Regularization: Explanation Accuracy (Case 1 and Case 2) for QGLIME (HSIC-G) at varying \(\lambda\)}
    \begin{tabular}{lcc|cccc}
        \toprule
        \(\lambda\) & \emph{One@1} & \emph{One@3} & \emph{Both@2} & \emph{Both@6} & \emph{One@2} & \emph{One@6} \\
        \midrule
        1  & $\best{1.00 \pm 0.00}$ & $\best{1.00 \pm 0.00}$ & $0.10 \pm 0.30$ & $0.20 \pm 0.40$ & $\best{0.80 \pm 0.40}$ & $\best{1.00 \pm 0.00}$ \\
        $10^{-1}$ & $\best{1.00 \pm 0.00}$ & $\best{1.00 \pm 0.00}$ & $0.10 \pm 0.30$ & $0.20 \pm 0.40$ & $\best{0.80 \pm 0.40}$ & $\best{1.00 \pm 0.00}$ \\
        $10^{-2}$ & $\best{1.00 \pm 0.00}$ & $\best{1.00 \pm 0.00}$ & $\best{0.80 \pm 0.40}$ & $\best{0.80 \pm 0.40}$ & $\best{0.80 \pm 0.40}$ & $\best{1.00 \pm 0.00}$ \\
        $10^{-3}$ & $\best{1.00 \pm 0.00}$ & $\best{1.00 \pm 0.00}$ & $\best{0.80 \pm 0.40}$ & $\best{0.80 \pm 0.40}$ & $\best{0.80 \pm 0.40}$ & $\best{1.00 \pm 0.00}$ \\
        $10^{-4}$ & $\best{1.00 \pm 0.00}$ & $\best{1.00 \pm 0.00}$ & $\best{0.80 \pm 0.40}$ & $\best{0.80 \pm 0.40}$ & $\best{0.80 \pm 0.40}$ & $\best{1.00 \pm 0.00}$ \\
        \bottomrule
    \end{tabular}
    \label{tab:combined_accuracy_lambda}
\end{table}

\begin{table}[ht]
\centering
\caption{Effect of Regularization: Metrics (Case 1 and Case 2) for QGLIME (HSIC-G) at varying \(\lambda\)}
\begin{tabular}{lccccc}
    \toprule
    \(\lambda\) & $\fidm$ & $\fidp$  & $\spars$ & $\cons$ & $\ri$ \\
    \midrule
    \multicolumn{6}{l}{\textbf{Case 1: Single-Target}} \\
    \addlinespace[0.3em]
    $1$  & $0.734 \pm 0.324$ & $\best{0.988 \pm 0.025}$ & $0.409 \pm 0.418$ & $\best{1.0}$ & $0.10$ \\
    $10^{-1}$ & $\best{0.930 \pm 0.198}$ & $0.963 \pm 0.030$ & $\sbest{0.608 \pm 0.378}$ & $\best{1.0}$ & $7.39$ \\
    $10^{-2}$ & $\sbest{0.799 \pm 0.303}$ & $\sbest{0.975 \pm 0.030}$ & $\best{0.628 \pm 0.349}$ & $\best{1.0}$ & $\best{8.61}$ \\
    $10^{-3}$ & $0.798 \pm 0.302$ & $0.963 \pm 0.030$ & $0.568 \pm 0.426$ & $\best{1.0}$ & $\sbest{7.76}$ \\
    $10^{-4}$ & $0.798 \pm 0.302$ & $0.963 \pm 0.030$ & $0.568 \pm 0.426$ & $\best{1.0}$ & $7.76$ \\
    \addlinespace[0.2em]
    \midrule
    \multicolumn{6}{l}{\textbf{Case 2: Dual-Target}} \\
    \addlinespace[0.3em]
    $1$  & $0.102 \pm 0.299$ & $\best{0.375 \pm 0.401}$ & $0.000 \pm 0.000$ & $0.16$ & $0.10$ \\
    $10^{-1}$ & $0.780 \pm 0.385$ & $0.001 \pm 0.000$ & $\best{0.428 \pm 0.233}$ & $\best{0.26}$ & $0.35$ \\
    $10^{-2}$ & $0.765 \pm 0.378$ & $0.001 \pm 0.000$ & $\best{0.428 \pm 0.233}$ & $\sbest{0.22}$ & $0.69$ \\
    $10^{-3}$ & $\sbest{0.781 \pm 0.385}$ & $0.001 \pm 0.000$ & $\best{0.428 \pm 0.233}$ & $0.20$ & $\best{0.73}$ \\
    $10^{-4}$ & $\best{0.781 \pm 0.385}$ & $0.001 \pm 0.000$ & $\best{0.428 \pm 0.233}$ & $\sbest{0.22}$ & $\sbest{0.72}$ \\
    \bottomrule
\end{tabular}
\label{tab:combined_metrics_lambda}
\end{table}

Tables~\ref{tab:combined_accuracy_lambda} and \ref{tab:combined_metrics_lambda} summarize the effect of varying the HSIC-Group regularization parameter (\(\lambda\)) on QGraphLIME explanations in single-target and dual-target scenarios.

In the single-target case, explanatory performance remains robust across a wide range of \(\lambda\) values, indicating stable target identification and consistent fidelity. Even with moderate variation in \(\lambda\), sparsity and confidence metrics remain largely unaffected, suggesting that single-target explanations are resilient to regularization strength.

Conversely, in the dual-target scenario, weaker regularization leads to a noticeable drop in both accuracy and confidence, reflecting difficulty in capturing multiple relevant nodes simultaneously. Explanations become less sparse and more diffuse, compromising interpretability. Stronger regularization, by contrast, maintains focused, concise, and reliable explanations, underscoring its importance in complex multi-target settings. Empirically, performance plateaus for \(\lambda \ge 10^{-2}\), suggesting an optimal balance between surrogate regularization and explanation fidelity.

These findings collectively highlight that careful tuning of HSIC regularization is crucial for maintaining explanation quality, particularly when explaining multi-target QGNN decisions where interpretability, accuracy, and sparsity must all be preserved.

\section{Conclusion}

This work introduces \emph{QGraphLIME}, a model-agnostic framework for explaining Quantum Graph Neural Networks (QGNNs). The method fits an ensemble of HSIC-based local surrogates on locality-preserving graph perturbations, and aggregates their attributions and dispersion to produce stable, uncertainty-aware node or edge rankings under quantum measurement-induced stochasticity. The framework further derives a distribution-free, sharp sample-complexity bound for the surrogate ensemble via the Dvoretzky-Kiefer-Wolfowitz inequality, providing a finite-sample guarantee on the minimum number of surrogates required to uniformly approximate the measurement-induced distribution of a binary class probability at a prescribed accuracy and confidence under i.i.d.\ assumptions.

Empirically, across controlled synthetic graphs with known targets, QGraphLIME achieves high top-$k$ accuracy and label/probability fidelity in single-target tasks. Its uncertainty summaries—including Top-$k$ inclusion probability, interquartile ranges, and flip probabilities—highlight regions of stability and indecision. Ablation studies demonstrate that nonlinear HSIC surrogates and appropriately chosen perturbations materially enhance explanation reliability relative to linear baselines. These results underscore the necessity of quantum-specific explainability for QGNNs: nonlinear HSIC-based surrogates effectively capture quantum-induced dependencies, while ensemble-based uncertainty metrics provide robust, multi-dimensional measures of explanation reliability, enhancing trust in high-stakes quantum applications.

\subsection*{Limitations}

Despite distribution-free guarantees and strong empirical performance, several limitations remain. First, perturbation design significantly affects results: random-walk perturbations, while structure-aware, can reduce explanation completeness in multi-target settings by overemphasizing local connectivity, whereas random node perturbations better capture all influential nodes at the cost of slightly lower per-node confidence. This illustrates a task-dependent trade-off between locality preservation and comprehensive coverage of causal substructures.

Second, this study focuses exclusively on Equivariant Quantum Graph Circuits (EQGCs) as the target QGNN family. Extending evaluations to alternative architectures and readout strategies is necessary to establish the generality of the uncertainty bounds and the behavior of HSIC-based surrogates under different parameter-sharing schemes and measurement observables. Moreover, experiments are conducted on small synthetic datasets with explicit ground-truth targets due to NISQ-era constraints and the exponential cost of classical statevector simulations. Consequently, external validity is limited, and scaling to larger, real-world graphs is required to assess robustness under realistic structural heterogeneity and hardware noise.

Third, the theoretical guarantees assume i.i.d.\ surrogate-level statistics. In practice, correlated randomness or hardware drift can reduce effective sample sizes, making the DKW-based ensemble budget conservative. Dependence diagnostics and adjusted ensemble sizing may be necessary. Finally, fidelity under node retention or removal depends on masking semantics and graph reindexing, which can influence message-passing statistics. Although a consistent protocol is applied here, alternative intervention choices could modify measured fidelity metrics. More broadly, classical surrogates cannot directly capture intrinsically quantum properties such as phase, entanglement, and non-commuting measurements, motivating future extensions that incorporate quantum-aware features and error-mitigated measurement schemes.

\subsection*{Future Work}

Future research will prioritize deploying QGraphLIME on real quantum hardware to evaluate performance under realistic noise conditions and integrating it with quantum error mitigation techniques to enhance reliability. Benchmarking on real-world datasets will provide critical insights into the framework's scalability and practical applicability in authentic quantum machine learning scenarios. Additionally, extending evaluations across diverse QGNN architectures will help establish the generality and robustness of QGraphLIME across different model families. On the theoretical front, efforts will focus on deriving tighter uncertainty bounds and refining surrogate model estimation guarantees, thereby strengthening the formal foundations of quantum explainability. Collectively, these directions aim to advance QGraphLIME toward a fully deployable and theoretically rigorous framework for interpretable quantum graph learning.

\bibliographystyle{tmlr}
\bibliography{qgxai}

\end{document}